\documentclass[10pt,twocolumn,twoside]{IEEEtran}

\usepackage{amsmath}
\usepackage{amssymb}
\usepackage{amsbsy} 
\usepackage{amsthm} 
\usepackage{bbm}
\usepackage{url}
\usepackage{array}
\usepackage{color}
\usepackage{multirow}
\usepackage{epsfig}
\usepackage{colortbl}
\usepackage [table]{xcolor}
\usepackage{graphicx}
\usepackage{float}
\usepackage{caption}
\usepackage{algorithm}
\usepackage{algpseudocode} 
\usepackage{mathtools} 
\usepackage{siunitx} 
\usepackage{subfigure}
\usepackage{booktabs} 

\DeclareMathOperator*{\argmin}{arg\,min}
\DeclareMathOperator*{\minimize}{minimize}

\newcommand{\ts}{\textsuperscript}
\newcommand{\pluseq}{\mathrel{+}=}

\floatname{algorithm}{Algorithm} 

\newtheorem{theorem}{Theorem}[section]
\newtheorem{lemma}[theorem]{Lemma}

\newtheorem{case}{Case}

\begin{document}

\title{Joint System and Algorithm Design for Computationally Efficient Fan Beam Coded Aperture X-ray Coherent Scatter Imaging}

\author{Ikenna~Odinaka,
		Joseph~A.~O'Sullivan,
		David~G.~Politte,
		Kenneth~P.~MacCabe,
		Yan~Kaganovsky,
		Joel~A.~Greenberg,
		Manu~Lakshmanan,
		Kalyani~Krishnamurthy,
		Anuj~Kapadia,
		Lawrence~Carin,
		and~David~J.~Brady
		
\thanks{This paper has been submitted to IEEE Transactions on Computational Imaging for consideration.}

\thanks{Ikenna~Odinaka (ikenna.odinaka@duke.edu), Yan~Kaganovsky, Joel~A.~Greenberg, Lawrence~Carin, and David~J.~Brady are with the Department of Electrical and Computer Engineering, Duke University, Durham, NC 27708.}

\thanks{Joseph~A.~O'Sullivan is with the Department of Electrical and Systems Engineering, Washington University in St. Louis, Saint Louis, MO 63130}

\thanks{David~G.~Politte is with the Mallinckrodt Institute of Radiology, Washington University in St. Louis, Saint Louis, MO 63110}

\thanks{Kenneth~P.~MacCabe is with Physical Numerics, Raleigh, NC 27675}

\thanks{Manu~Lakshmanan is with the Department of Biomedical Engineering, Duke University, Durham, NC 27708}

\thanks{Kalyani~Krishnamurthy is with Pendar Medical LLC, Cambridge, MA 02138}

\thanks{Anuj~Kapadia is with the Department of Radiology, Duke University Medical Center, Durham, NC 27710}

\thanks{K.P. MacCabe and K. Krishnamurthy were with the Department of Electrical and Computer Engineering, Duke University, Durham, NC 27708, when this work was done.}
 
\thanks{This work has been supported by the U.S. Department of Homeland Security Science and Technology Directorate Contract Number: HSHQDC-11-C-00083}
}

\maketitle

\begin{abstract}
In x-ray coherent scatter tomography, tomographic measurements of the forward scatter distribution are used to infer scatter densities within a volume. A radiopaque 2D pattern placed between the object and the detector array enables the disambiguation between different scatter events. The use of a fan beam source illumination to speed up data acquisition relative to a pencil beam presents computational challenges. To facilitate the use of iterative algorithms based on a penalized Poisson log-likelihood function, efficient computational implementation of the forward and backward models are needed. Our proposed implementation exploits physical symmetries and structural properties of the system and suggests a joint system-algorithm design, where the system design choices are influenced by computational considerations, and in turn lead to reduced reconstruction time. Computational-time speedups of approximately 146 and 32 are achieved in the computation of the forward and backward models, respectively.  Results validating the forward model and reconstruction algorithm are presented on simulated analytic and Monte Carlo data.
\end{abstract}

\section{Introduction}  \label{sec:introduction}
X-ray coherent scatter imaging involves the reconstruction of an object's volumetric scatter density from tomographic measurements of the scattered x-ray data. The forward scatter distribution from an object is modeled as a superposition of the intensities of the scatter from all object points. The scatter from each object point occurs due to photons illuminating the object from all incident angles at different energy-dependent intensities, as permitted by the source configuration. The detailed forward model is given below, but it is immediately apparent that the central challenge in x-ray coherent scatter imaging is the separation of the scatter back into a volumetric scatter density. 

\begin{figure*}%
\centering
	\subfigure[Fan beam geometry schematic]{%
		\includegraphics[width=0.6\textwidth]{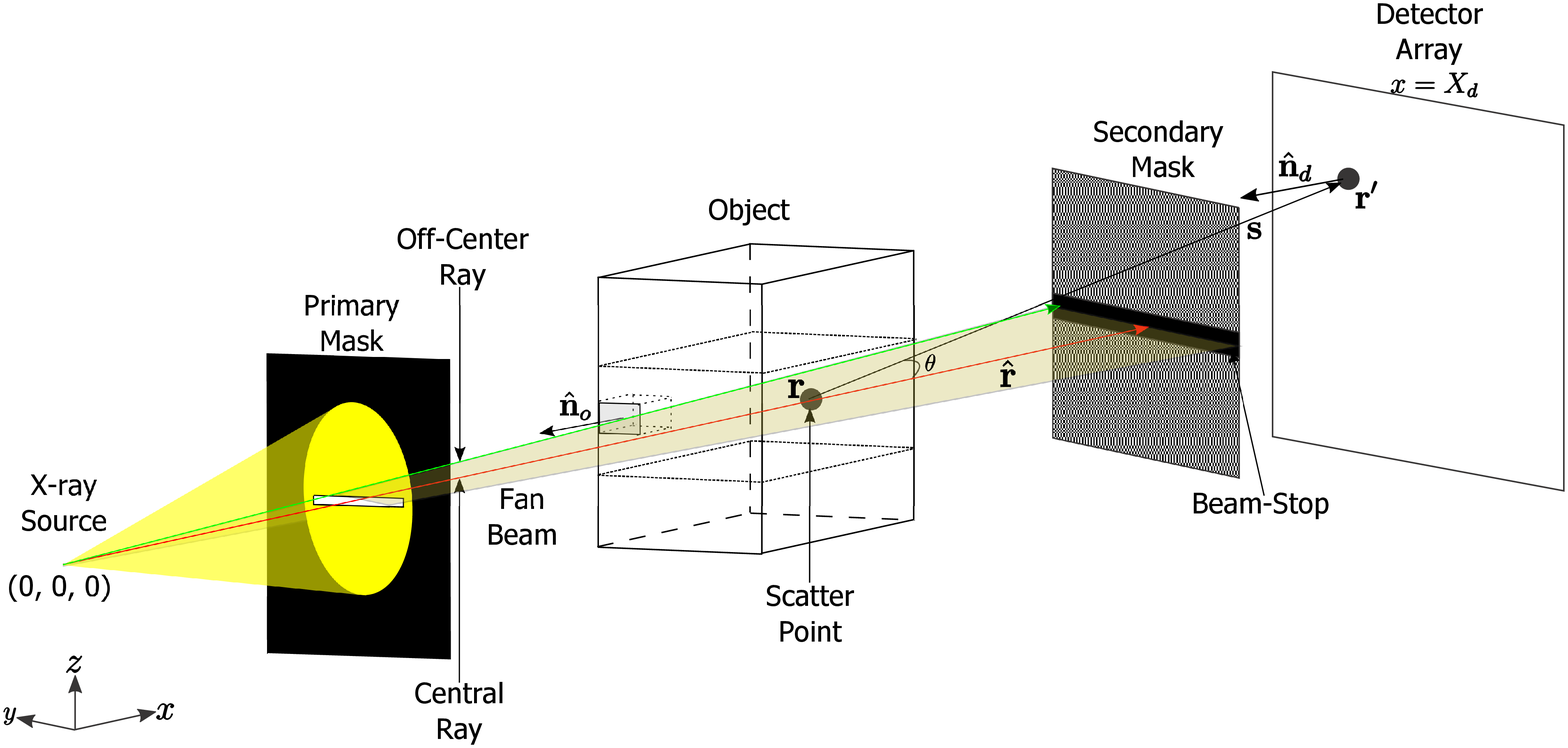}%
		\label{fig:fanbeam_schematic}}%
	\\
	\subfigure[Secondary mask image]{%
		\includegraphics[width=0.3\textwidth]{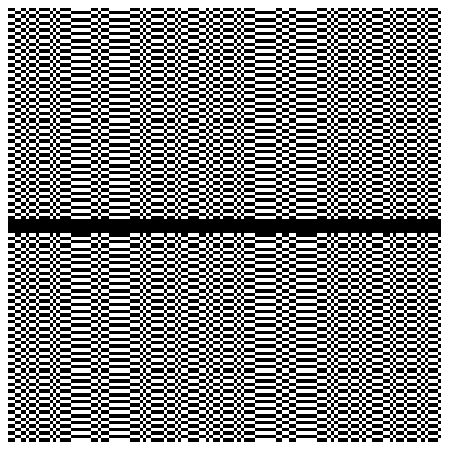}%
		\label{fig:secondary_mask_MURI}}%
	~
	\subfigure[Disambiguating scatter]{%
		\includegraphics[width=0.5\textwidth]{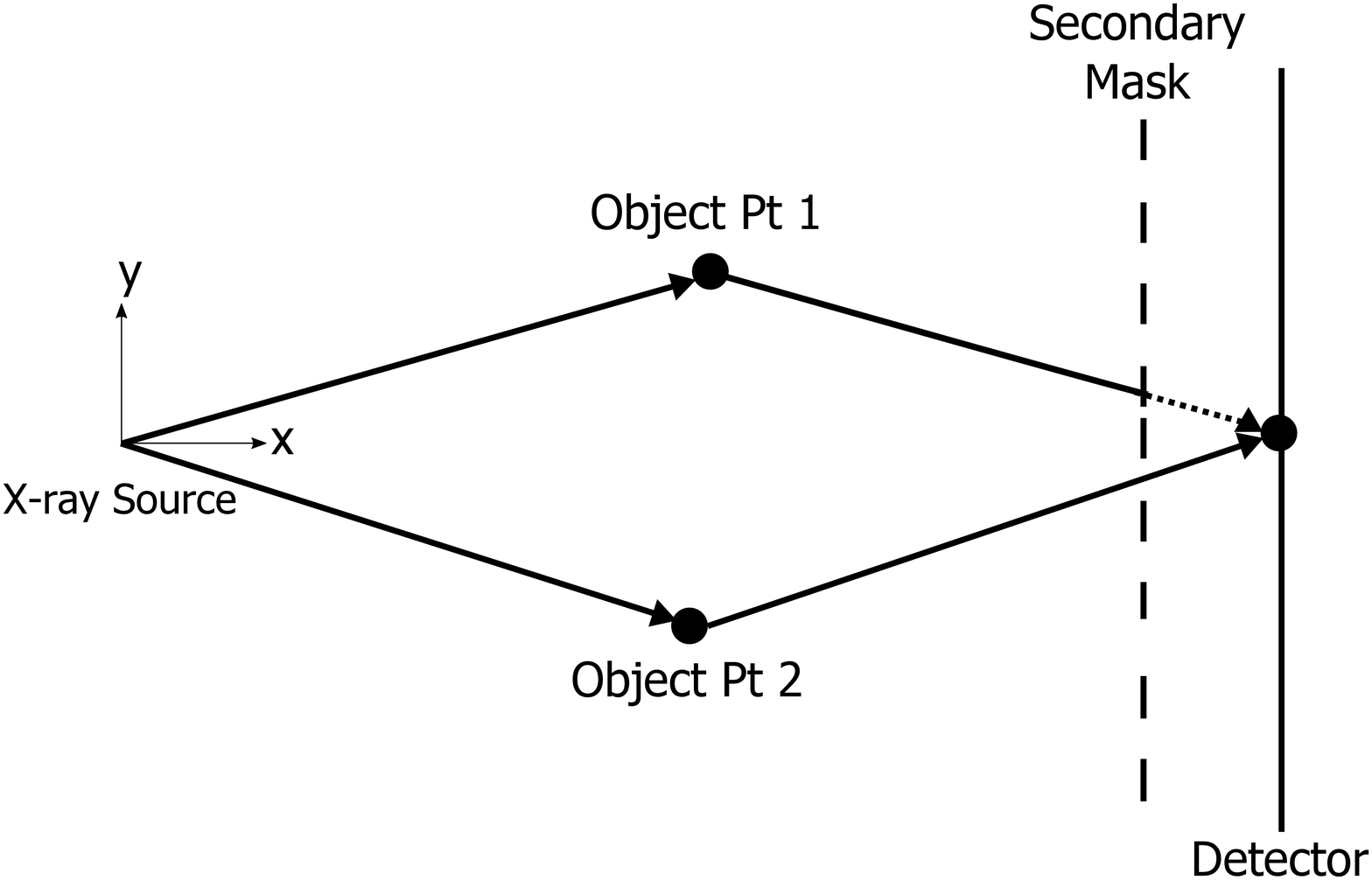}%
		\label{fig:mask_in_action}}%
	\\
	\subfigure[Translation symmetry]{%
		\includegraphics[width=0.5\textwidth]{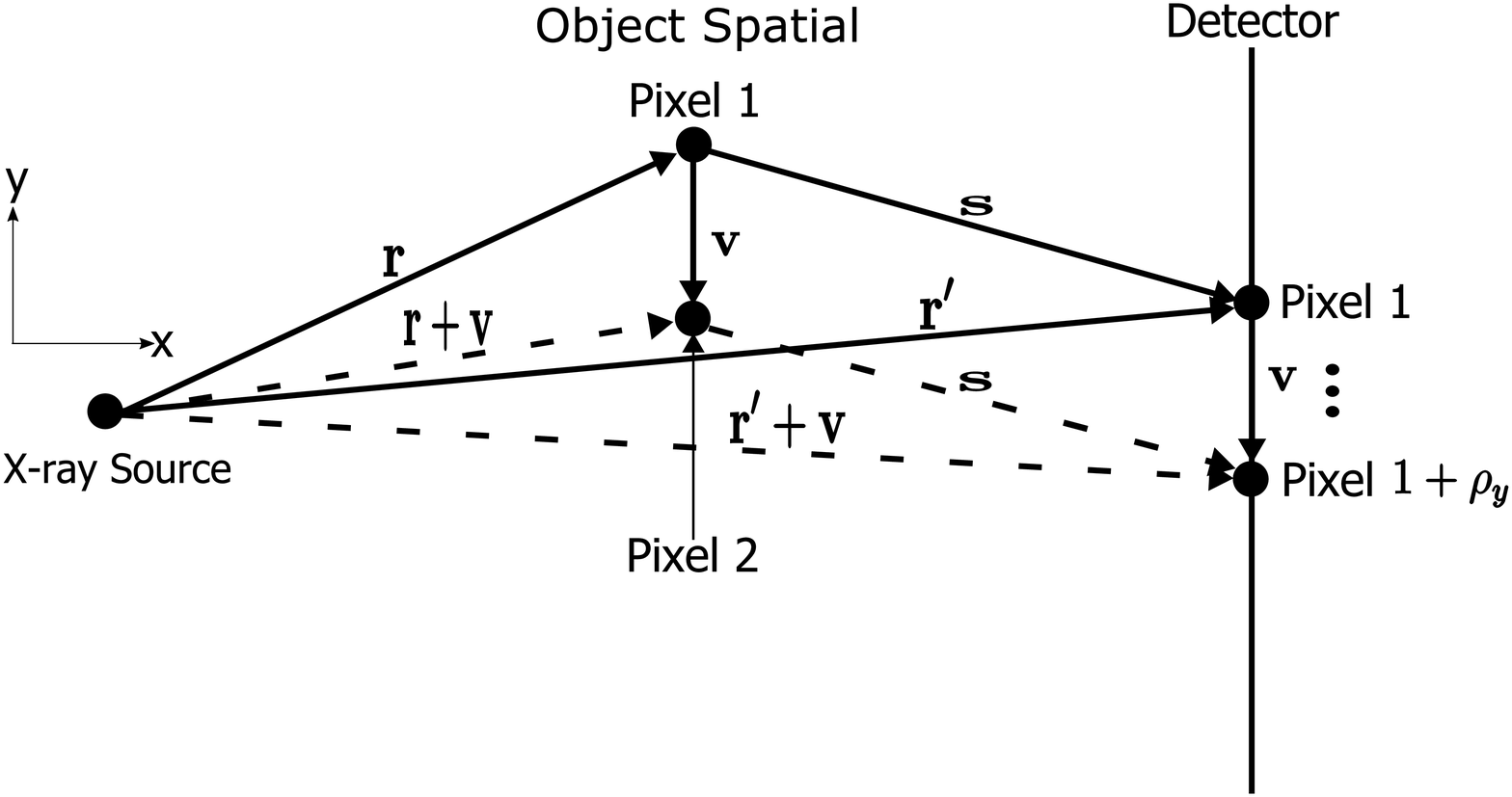}%
		\label{fig:translation_symmetry}}%
	\caption{In (a), the focal spot of the source is located at the origin of the coordinate system, locations in the object volume and detector array are given as $\mathbf{r}$ and $\mathbf{r}^{\prime}$, respectively. The scatter angle and scatter vector are denoted by $\theta$ and $\mathbf{s}$, respectively. $\hat{\mathbf{r}}$ is a unit vector along the direction from the source to an object point, $\hat{\mathbf{n}}_o$ is a unit normal vector to the face of the rectangular object slice illuminated by the fan beam, and $\hat{\mathbf{n}}_d$ is a unit normal vector to the front of the detector array. The rectangular object slice is exaggerated to illustrate the unit normal vector. The secondary mask in the fan beam geometry schematic is isolated in (b). (c) shows an example of the secondary mask differentiating scatter from different object locations to the same detector location. The scattered photons from object point 1 are blocked by the secondary mask, while those from object point 2 are transmitted. A single slice through the secondary mask and the detector array, parallel to the y-axis, is depicted in (c). (d) shows the effect of translation symmetry on object pixel width selection. $\mathbf{v}$ is a vector corresponding to the translation of an object pixel by one pixel width along the $y$ direction. $\rho_y$ is the translation symmetry step size.}%
	\label{fig:fanbeam_setup}%
\end{figure*}

One approach to separate the scatter from the multiplexed measurements is to obtain different views of the object by using a rotating gantry \cite{Schloraka2002}. An alternative and novel approach is the use of a coded aperture between the object and the detector \cite{Brady2013coding, Brady2015}, which prevents the need for a rotating gantry. Two important design elements in coded aperture x-ray coherent scatter imaging are a primary mask between the source and the object, and a coded aperture (secondary mask) between the object and the detector array. Whereas the primary mask serves to shape the incident beam, the secondary mask is responsible for blocking the transmitted x-rays and disambiguating scatter angles. 

Previous work in coded aperture x-ray coherent scatter imaging has shown the success of a pencil beam system (primary mask with a single small hole) \cite{Maccabe2012pencil}. We scale up to a fan beam system (primary mask with a slit) to accelerate measurement acquisition (see Fig.~\ref{fig:fanbeam_schematic}). For this, new codes and algorithms are needed for efficient data collection and inference. Brady \textit{et al.} \cite{Brady2013coding} give a detailed description of the design of new codes for a fan beam system. Here, we focus on efficient data inversion. MacCabe \textit{et al.} \cite{Maccabe2013snapshot} analyzed the fan beam system with a 2D energy-integrating detector array and a 2D secondary mask (see Fig.~\ref{fig:secondary_mask_MURI}) from the standpoint of recovering the object's spatial and angular distributions, under the assumption that the object may be factorized in space and angle. Moreover, their system matrix was small enough to be stored in main memory on a standard computer. In this paper, we focus on general spatial and spectral (momentum transfer) distributions and we consider a much larger system, where the storage of the system matrix may not be feasible. Unlike the angular distribution, the spectral distribution can be used as a unique signature for identifying materials \cite{Greenberg2015}.

Coded aperture design is based upon a combination of objectives. As many photons as possible should be allowed to hit the detector to increase the signal-to-noise ratio. However, a smaller number of holes in an aperture may allow for a more accurate determination of object location and scatter angle. For example, if the aperture has a single hole, the illuminated detector points and the aperture hole define scattered rays that can be traced back to unique object locations. Designs of coded apertures range from intuitive to analytical. On the analytical side, there is a desire for a uniform sensitivity to a wide range of scatter locations and scatter angles. Such analytical designs are described elsewhere \cite{Brady2013coding}. There is also a computational imaging aspect of the secondary mask design. The secondary mask and the inference algorithm may be designed jointly, with a goal of optimizing a measure of the reconstruction performance. In this view, the ad hoc intuitive arguments for secondary mask selection play a subsidiary role to quantitative measures of performance. While this is our ultimate goal, this paper focuses on inference algorithm development using a penalized maximum likelihood estimation approach.

The inference algorithm relies critically on a computational representation of the forward model for the data. Our forward model has been derived analytically and tested using analytic and Monte Carlo simulations. We propose an efficient representation of the forward model based upon physical symmetries and structural properties (e.g. smoothness considerations) of the system, described in detail below. The physical symmetries assumptions are readily verifiable in analytical and Monte Carlo models, but less so in experiments. The ultimate benefit of the model and approach described here will be determined in part by the calibration process for verifying alignment and symmetry. Even if the symmetry assumptions break down in practice, other aspects of the efficient computation of the forward model such as smoothness consideration may be valid. The computational implementation also involves a trade-off between on-line and off-line computations which can be optimized for a target computer. 
 
There is a connection between the computational time of the iterative algorithm and judicious choices for the sampling in the image domain. If the voxel size is chosen to be an integer multiple of the spacing of pixels on the detector, many geometry computations can be reused. A common technique for accelerating the convergence of an iterative algorithm involves partitioning the detector measurements and using each partition (subset) in turn to update the estimated image parameters. This technique is called ordered subsets (OS). The symmetries available in the ideal case motivate a particular choice of subsets in an ordered subset expectation-maximization OSEM-type algorithm.

Reconstruction results for simulated data based on the analytic forward model and Monte Carlo simulation demonstrate the algorithmic performance for a particular choice of the secondary mask. Computational-time speedups of about 146 and 32 are obtained for the forward and backward models, respectively. In addition, the spatial distribution and the momentum transfer profiles of the simulated object are recovered quickly and fairly accurately, with only a few ordered subset EM-type iterations required.

The remainder of this paper is organized as follows: In Section~\ref{sec:forward_model}, we describe the forward model. Sections~\ref{sec:computation_forward_model} and \ref{sec:computation_backward_model} describe efficient computations of the forward and backward models, respectively. The measurement model and image recovery problem are described in Section~\ref{sec:image_recovery_problem_desc}, while Section~\ref{sec:reconstruction_algorithm} develops the regularized ordered subset EM-type reconstruction algorithm. The results of applying the ordered subset EM-type algorithm to analytically and Monte Carlo simulated data are presented in Section~\ref{sec:results}, while concluding remarks are given in Section~\ref{sec:conclusions}. Appendices~\ref{app:forward_model} and \ref{app:reconstruction_algorithm_details} describe, in detail, the derivation of the forward model and the image recovery algorithm, respectively.

\section{Forward Model} \label{sec:forward_model}
In the fan beam coded aperture coherent scatter model, the x-ray source transmits photons in a fan within a plane. As shown in Fig.~\ref{fig:fanbeam_schematic}, each photon illuminates and interacts with the object. The photons are either absorbed, transmitted, or scattered by the object. The transmitted and scattered photons propagate from the object to the detector array and are either blocked or transmitted by the intervening secondary mask. 

As shown in Fig.~\ref{fig:fanbeam_schematic}, the focal spot of the source is located at the origin of the coordinate system and the fan beam resides in the $z = 0$ plane. Object scatter locations are indexed by $(x, y)$ coordinates and corresponding object point $\mathbf{r} = [x, y, 0]$, with positive $x$-coordinate pointing from the source through the object to the secondary mask and the detector. Since the source is located at the origin, $\mathbf{r}$ also serves as the source to object point vector, with magnitude $r$ and unit vector given as $\hat{\mathbf{r}}$.  The detector plane is perpendicular to the central ray of the x-ray fan which goes along the x-axis. A point on the detector is $\mathbf{r}^{\prime} = [X_d, y^{\prime}, z^{\prime}]$. The unit normal vector to the detector plane is $\hat{\mathbf{n}}_d$ as shown in Fig.~\ref{fig:fanbeam_schematic}. The object point and the detector point determine the scatter vector $\mathbf{s} = \mathbf{r}^{\prime} - \mathbf{r}$, with magnitude $s$ and unit vector $\hat{\mathbf{s}}$. The scatter vector and the source to object point vector determine the scatter angle $\theta$, shown in Fig.~\ref{fig:fanbeam_schematic}. 

The flux measured by the energy-integrating detector array at location $\mathbf{r}^{\prime}$ is given as 

\begin{equation}
g(\mathbf{r}^{\prime}) = \int \int \, H\left(\mathbf{r}^{\prime}, \mathbf{r}, q\right) f(\mathbf{r}, q) d\mathbf{r} \, dq,
\label{eq:forward_model_energy_integrating}
\end{equation}
\noindent{where}
\begin{eqnarray}
H\left(\mathbf{r}^{\prime}, \mathbf{r}, q\right) & = & C G_{s o}(\mathbf{r}) G_{o d}(\mathbf{s}) T(\mathbf{r}, \mathbf{r}^{\prime}) \Delta\theta S(\theta, q) \label{eq:forward_operator_energy_integrating} \\
G_{s o}(\mathbf{r}) & = & \frac{x}{(x^2 + y^2)^{1.5}} \nonumber \\
G_{o d}(\mathbf{s}) & = & \frac{\left|\hat{\mathbf{n}}_d \cdot \hat{\mathbf{s}}\right|}{s^2}.\nonumber \\
S(\theta, q) & = & \frac{q \left(1 + \cos^2\theta\right) \cos\left(\frac{\theta}{2}\right)}{\sin^2\left(\frac{\theta}{2}\right)} \Phi\left(\frac{h c q}{\sin\left(\frac{\theta}{2}\right)}\right)\nonumber.
\end{eqnarray}
\noindent{$H\left(\mathbf{r}^{\prime}, \mathbf{r}, q\right)$ is the forward operator, which includes geometric factors such as the source-to-object geometry factor $G_{s o}(\mathbf{r})$, the object-to-detector geometry factor $G_{o d}(\mathbf{s})$, the coded aperture mask modulation (transmission) factor $T(\mathbf{r}, \mathbf{r}^{\prime})$, and the scatter angle spread $\Delta\theta$.} The forward operator also includes a spectral factor $S(\theta, q)$ due to photon-matter interaction and a normalization constant $C$. $\Phi(E)$ is the spectra of the x-ray source, $h$ is Planck's constant, and $c$ is the speed of light. The object scatter density $f(\mathbf{r}, q)$, which is to be recovered, is a function of the object spatial location $\mathbf{r}$ and momentum transfer $q$, where $q$ is given in \si{\angstrom}$^{-1}$. See Appendix~\ref{app:forward_model} for details on the derivation of the multiplicative factors that comprise $H\left(\mathbf{r}^{\prime}, \mathbf{r}, q\right)$. Although the forward model is expressed using continuous variables and integrals, it is implemented discretely. 

The linear forward model given in Eq.~\ref{eq:forward_model_energy_integrating} ignores the self-attenuation of the incident source-to-object photons and scattered object-to-detector photons. Moreover, the model does not account for Compton scattering, which dominates coherent scattering at larger scatter angles. Ignoring self-attenuation and Compton scattering produces a linear model that is adequate for weakly attenuating objects and scatter measurements acquired at low scatter angles \cite{Harding1987, Greenberg2013snapshot}.

\section{Computation of the Forward Model} \label{sec:computation_forward_model}
In anticipation of an iterative algorithm for estimating the object scatter density, presented in Section~\ref{sec:image_recovery_problem_desc}, efficient computational implementations of the forward and backward models are required. The factors in Eq.~\ref{eq:forward_model_energy_integrating} may be reordered to avoid repeated computations as follows 

\begin{equation}
g(\mathbf{r}^{\prime})  =  \int \, C G_{s o}(\mathbf{r}) G_{o d}(\mathbf{s}) T(\mathbf{r}, \mathbf{r}^{\prime}) \Delta\theta W(\mathbf{r}, \theta) d\mathbf{r}, \label{eq:forward_model_energy_integrating_q_integrate}
\end{equation}
\noindent{where}
\begin{equation}
\label{eq:q_integral}
W(\mathbf{r}, \theta) = \int S(\theta, q) f(\mathbf{r}, q) \, dq
\end{equation}
\noindent{is the effective spectral factor. The integral in Eq.~\ref{eq:q_integral} takes scatter density, at a given object point, as a function of momentum transfer and produces scatter density as a function of scatter angle.} A straightforward implementation of the forward model in Eq.~\ref{eq:forward_model_energy_integrating_q_integrate} yields a computational structure given in Algorithm~\ref{alg:raw_forward_model_energy_integrating}. 

Depending on the programming language (scripting versus compiled) and the amount of main memory available, any subset of the $\mathbf{r}$, $\mathbf{r}^{\prime}$, and $q$ loops can be vectorized. For large problems, where storing the system matrix $H$ is not practical, all the loops cannot be vectorized. For such problems, most of the computations need to be performed online and efficiently, with as many computations as possible reused in the code. In this paper, we have chosen to vectorize only the $\mathbf{r}^{\prime}$ and $q$ loops. As a consequence of the vectorization of the $q$ loop, the innermost momentum transfer loop is implemented as a matrix-vector product. Note that the same vectorization approach applies to the non-optimized and accelerated versions of the forward model.

\begin{algorithm}
\caption{Computational structure for forward model with no optimization}\label{alg:raw_forward_model_energy_integrating}
\begin{algorithmic}
\State Given object scatter density $f(\mathbf{r}, q)$
\For{\textbf{each} object point $\mathbf{r}$}
	\For{\textbf{each} detector point $\mathbf{r}^{\prime}$}
		\State Compute $G_{s o}(\mathbf{r})$
		\State Compute $G_{o d}(\mathbf{s})$		
		\State Compute $T(\mathbf{r}, \mathbf{r}^{\prime})$	
		\State Compute $\Delta\theta$	
		\State Compute scatter angle $\theta(\mathbf{r}, \mathbf{r}^{\prime})$
		\For{\textbf{each} momentum transfer $q$}
			\State Compute $S(\theta, q)$
			\State $W(\mathbf{r}, \theta) \pluseq  S(\theta, q) f(\mathbf{r}, q)$
		\EndFor
		\State $\begin{aligned} 
			g(\mathbf{r}^{\prime}) &\pluseq G_{s o}(\mathbf{r}) \times G_{o d}(\mathbf{s}) \times 	T(\mathbf{r}, \mathbf{r}^{\prime}) \\ 				
								   &\qquad \times \Delta\theta \times W(\mathbf{r}, \theta)
			   \end{aligned}$						
	\EndFor
\EndFor
\State $g(\mathbf{r}^{\prime}) \mathrel{{\times}{=}} C$
\State Return detector image $g(\mathbf{r}^{\prime})$
\end{algorithmic}
\end{algorithm}

Three ways to improve the efficiency of the code are (1) scatter angle interpolation, (2) exploiting geometric symmetries in computing the system factors, and (3) balancing online and offline computations. The result of applying these elements of efficiency to the non-optimized forward model is given in Algorithm~\ref{alg:efficient_forward_model_energy_integrating}. The algorithm shows the final form of the forward model computational structure, incorporating offline computation of the spectral, source-to-object geometry, and mask modulation factors, interpolation of scatter angles, translation symmetry, and left-right and up-down mirror symmetries. We found that the most significant improvement in computational time is due to scatter angle interpolation. Further details are described below.

\begin{algorithm*}[!htp]
\caption{Computational structure for forward model incorporating the proposed optimizations. $\mathbf{f}$ and $\mathbf{g}$ are the set of values in the object domain and detector corresponding to $f(\mathbf{r}, q)$ and $g(\mathbf{r}^{\prime})$, respectively. $x \bullet y$ denotes the inner product between vectors $x$ and $y$, $\odot$ represents element-wise multiplication. $\operatorname{fliplr}(\cdot)$, $\operatorname{flipud}(\cdot)$, and $\operatorname{vertcat}(\cdot, \cdot)$ are mathematical operators for reversing a matrix column-wise from left to right, row-wise from top to bottom, and for vertically concatenating two matrices, respectively. $U$ and $D$ are labels for the up and down halves of the detector array, whereas $L$ and $R$ are labels for the left and right sides of the object domain. For example, $T_{UL}(\mathbf{r})$ is the mask transmission values for scatter vectors from a spatial location $\mathbf{r}$ on the left side of the object domain to the pixels in the upper half of the detector array. $M$ and $N$ are the number of pixels in the detector array along the $z$ and $y$ directions, respectively. $G_{o d}(1: M/2, 1: N)$ represents a matrix of object-to-detector geometry factors from a given object voxel to half of the detector array along $z$, with $M/2$ rows and $N$ columns. $1: N$ is a list of numbers from $1$ to $N$ in increments of $1$. $\rho_y$ is the translation symmetry step size corresponding to the number of detector pixels along the $y$ direction equivalent to the width of an object pixel along $y$. $S(\theta, :)$ is a vector of spectral factors over all $q$ for a fixed scatter angle $\theta$ while $f(\mathbf{r}, :)$ is a vector of $f(\mathbf{r}, q)$ values over all $q$ for a fixed $\mathbf{r}$. $C$ is the normalization constant.}\label{alg:efficient_forward_model_energy_integrating}
\begin{algorithmic}[1] 
\State \textbf{Precompute} $G_{s o}(\mathbf{r})$ for object spatial locations $\mathbf{r}$ \label{lst:forw:on_off_1}
\State \textbf{Precompute} $S(\theta, q)$ for predefined $q$ and $\theta$ samples \label{lst:forw:sai_1}
\State \textbf{Precompute} four symmetry-based mask factor images: $T_{UL}(\mathbf{r})$, $T_{DL}(\mathbf{r})$, $T_{UR}(\mathbf{r})$, $T_{DR}(\mathbf{r})$\label{lst:forw:on_off_2}
\Procedure{$\mathbf{g} =$ Forward\_project}{$\mathbf{f}$}%
\For{\textbf{each} object location along $x$} \label{lst:forw:x_dir}
	\For{\textbf{each} object location along $y$ in the left half} \label{lst:forw:y_dir_lr_sym_1}
		\State $W_L(\theta) = S(\theta, :) \bullet f([x, y, 0], :)$ 
		\State $W_R(\theta) = S(\theta, :) \bullet f([x, -y, 0], :)$ 	
		\If{first location along the $y$ direction} \label{lst:forw:t_sym_1}
			\State \textbf{Compute} $G_{o d}(1:M/2, 1:N)$ \label{lst:forw:ud_sym_1}
		\Else
			\State \textbf{Update} $G_{o d}(1:M/2, \rho_y+1:N) = G_{o d}(1:M/2, 1:N-\rho_y)$ 
			\State \textbf{Recompute} $G_{o d}(1:M/2, 1:\rho_y)$ 
			\EndIf \label{lst:forw:t_sym_2}
			\State Compute $\Delta\theta$  
			\State Compute scatter angles $\theta(\mathbf{r}, \mathbf{r}^{\prime})$  
			\State \textbf{Interpolate} $W_L(\theta)$ and $W_R(\theta)$ and store in $W_L^{\text{interp}}$ and $W_R^{\text{interp}}$, respectively \label{lst:forw:sai_2} 
			\State \textbf{Fetch} $G_{s o}([x, y, 0])$ \label{lst:forw:on_off_3}
			\State \textbf{Fetch} $T_{UL}([x, y, 0])$, $T_{DL}([x, y, 0])$, $T_{UR}([x, -y, 0])$, and $T_{DR}([x, -y, 0])$ \label{lst:forw:on_off_4}
			\State $\mathbf{g}_{UL} \pluseq G_{s o}([x, y, 0]) \odot G_{o d} \odot T_{UL}([x, y, 0]) \odot \Delta\theta \odot W_L^{\text{interp}}$ 
			\State $\mathbf{g}_{DL} \pluseq G_{s o}([x, y, 0]) \odot G_{o d} \odot T_{DL}([x, y, 0]) \odot \Delta\theta \odot W_L^{\text{interp}}$ 
			\State $\mathbf{g}_{UR} \pluseq G_{s o}([x, y, 0]) \odot G_{o d} \odot T_{UR}([x, -y, 0]) \odot \Delta\theta \odot W_R^{\text{interp}}$
			\State $\mathbf{g}_{DR} \pluseq G_{s o}([x, y, 0]) \odot G_{o d} \odot T_{DR}([x, -y, 0]) \odot \Delta\theta \odot W_R^{\text{interp}}$ 
		\EndFor
\EndFor
\State $\mathbf{g}_U = \mathbf{g}_{UL} + \operatorname{fliplr}(\mathbf{g}_{UR})$ \label{lst:forw:lr_sym_2}
\State $\mathbf{g}_D = \mathbf{g}_{DL} + \operatorname{fliplr}(\mathbf{g}_{DR})$ \label{lst:forw:lr_sym_3}
\State $\mathbf{g} = C \odot \operatorname{vertcat}(\mathbf{g}_U, \operatorname{flipud}(\mathbf{g}_D))$ \label{lst:forw:ud_sym_2}
\EndProcedure
\end{algorithmic}
\end{algorithm*}

\subsection{Geometric Symmetries}
The object-to-detector geometry factor $G_{o d}(\mathbf{s})$ in Eq.~\ref{eq:forward_operator_energy_integrating} relies on the computation of the scatter vector $\mathbf{s} = \mathbf{r}^{\prime} - \mathbf{r}$, which has length $s = \left|\mathbf{s}\right|$ and the cosine factor at the detector surface $\left|\hat{\mathbf{n}}_d \cdot \hat{\mathbf{s}}\right|$, where $\hat{\mathbf{s}} = \frac{\mathbf{s}}{s}$. For the geometry illustrated in Fig.~\ref{fig:fanbeam_schematic}, the cosine factor equals the magnitude of the $x$ component of $\hat{\mathbf{s}}$. We now consider three different ways to exploit the symmetries present in the system geometry. We assume that the incident fan beam lies in a plane that is perpendicular to the detector plane and that the central ray of the fan beam intersects the center of the detector array, as shown in Fig.~\ref{fig:fanbeam_schematic}. To illustrate the symmetries, we replace the loop over $\mathbf{r}$ in Algorithm~\ref{alg:raw_forward_model_energy_integrating} with a loop over the $x$ and $y$ directions in the object domain, as shown in lines~\ref{lst:forw:x_dir} and \ref{lst:forw:y_dir_lr_sym_1} of Algorithm~\ref{alg:efficient_forward_model_energy_integrating}.

\subsubsection{Translation Symmetry}
For the system envisioned in this paper, there is a large array of small detector pixels. A typical detector may have an array of approximately 1500 by 2000 pixels, with pixel widths of 0.19 mm. One choice in the reconstruction is the selection of pixel widths in the object domain. These pixel widths should be motivated from a first principles analysis of achievable object resolution. However, some flexibility in the exact pixel width remains. Suppose that the displacement between object pixel centers in the $y$ direction (across the fan beam) is an integer multiple of the detector pixel width. For our study, we use a factor of 16 to get 16*0.19 = 3.04 mm width in the $y$ direction. Let $\mathbf{v}$ be a vector corresponding to the translation of a pixel in the object by one pixel width along the $y$ direction. Then, the equality (see Fig.~\ref{fig:translation_symmetry})
\begin{equation}
\nonumber
\left(\mathbf{r}^{\prime} + \mathbf{v}\right) - \left(\mathbf{r} + \mathbf{v}\right) = \left(\mathbf{r}^{\prime} - \mathbf{r}\right) = \mathbf{s}
\end{equation}
\noindent{implies that for each scatter vector from object location $\mathbf{r}$ to detector location $\mathbf{r}^{\prime}$, there is an identical scatter vector from object location $\mathbf{r} + \mathbf{v}$ to detector location $\mathbf{r}^{\prime} + \mathbf{v}$. Thus, given scatter vectors computed for one value of $y$ in the object, the scatter vectors for adjacent object locations are nearly all determined, with only scatter vectors corresponding to detector pixels near the edge of the detector array needing recomputation. This gives an efficient update for the scatter vector computation and consequently, the object-to-detector geometry factor $G_{o d}$ computation.}

Let us assume there are $M/2$ rows in each half of the detector array along the $z$-direction, $N$ detector columns ($y$ direction), and $\rho_y$ represents the number of $y$-directional detector pixels that correspond to 1 $y$-directional object pixel. Lines \ref{lst:forw:t_sym_1} to \ref{lst:forw:t_sym_2} in Algorithm~\ref{alg:efficient_forward_model_energy_integrating} show the modification of the non-optimized forward model computation due to translation symmetry. For the first location in the object along the $y$ direction, the object-to-detector geometry factor $G_{o d}$ is computed for all detector pixels under consideration. For the next object pixel along $y$, the current values of $G_{o d}$ starting from column $\rho_y+1$ to $N$ are identical to the previous values of $G_{o d}$ starting from the first column to the last but $\rho_y$ column. The current values of $G_{o d}$ for the first $\rho_y$ columns need to be recomputed, since they have no precomputed correspondence. From the algorithm, we see that previous object-to-detector geometry factor computations are re-used, and only a small fraction of the factor needs to be recomputed. For a 1500 by 2000 detector array and a translation symmetry step size of 16, the reduction in object-to-detector geometry factor computation is approximately 98.9\%. In practice, the actual time savings obtained depends on the difference between the time it takes to compute the object-to-detector geometry factor and the time it takes to read it from main memory.

\subsubsection{Left-Right Mirror Symmetry}
Another form of symmetry that we can exploit in the computation of the forward model is left-right mirror symmetry. Since the central ray of the fan beam intersects the center of the detector array, the $y$-coordinates of the focal point of the source and the center of the detector array are equal. We can consider an object reconstruction region whose center's $y$-coordinate is aligned with that of the source and detector array. For such an object region, we can select an even number of pixels along the $y$ direction. With this choice, the scatter vectors from the left half of the object region are a mirror reflection of those on the right half of the region. The scatter angles and magnitude of the scatter vectors are equal for both halves of the object region. 

As a result of the left-right mirror symmetry, we only need to compute the scatter angles $\theta$, geometric factors $G_{s o}$, $G_{o d}$, and $\Delta\theta$, and the spectral factor $S(\theta, q)$ for one half of the object region, giving a factor of two speedup in computation, barring the costs of computing the mask transmission geometric factor $T$. Line \ref{lst:forw:y_dir_lr_sym_1} in Algorithm~\ref{alg:efficient_forward_model_energy_integrating} signals the beginning of exploiting left-right mirror symmetry. Lines \ref{lst:forw:lr_sym_2} and \ref{lst:forw:lr_sym_3} are needed to combine the left-right symmetric results of the forward model. 

\subsubsection{Up-Down Mirror Symmetry}
In addition to left-right mirror symmetry, further improvements in the computational efficiency of the forward model are possible. Since the central ray of the fan beam intersects the center of the detector array, the $z$-coordinate of the focal point of the source equals that of the detector array center. To simplify the discussion, we assume the detector array has an even number of pixels along $z$ (see Fig.~\ref{fig:OS_choice}), so that the $z$-coordinate of the center lies at the border between adjacent pixels. Consequently, the scatter vectors from the object to the top portion of the detector array are reflections of those to the bottom half. 

Since only half of the detector arrays are used in the computation of the scatter angle, geometric factors $G_{s o}$, $G_{o d}$, and $\Delta\theta$, and the spectral factor $S$, a further factor of two speedup in computational may be achieved, ignoring the cost of computing the mask factor $T$. The use of up-down mirror symmetry in Algorithm~\ref{alg:efficient_forward_model_energy_integrating} begins with line~\ref{lst:forw:ud_sym_1}. Line \ref{lst:forw:ud_sym_2} is needed to combine the top and bottom halves of the results of the forward model.

One of the difficulties with left-right and up-down mirror symmetries is that they are easily violated in practice. The fan beam may not lie in a plane perpendicular to the detector plane and the central ray of the fan beam may not intersect the center of the detector array. If these deviations are minor, rotation and interpolation may be utilized to reinstate the perpendicularity between the fan beam and detector planes. Moreover, the detector plane measurements may be transformed by shifting the measurement window along $y$ or $z$ so as to align the source and detector array center, with potential data loss due to cropping. In addition, if the object is not located around the central ray of the fan beam, an excess amount of pixels may be needed to cover the illuminated slice of the object. Given the potential significant speedup in the computation, care should be taken in the design and calibration of the coherent scatter imaging system to ensure that the assumptions for left-right and up-down mirror symmetries hold.

Speedups due to left-right and up-down mirror symmetries may be obtained in the computation of the mask modulation geometry factor $T$. This computational savings hinges on the fact that the secondary mask is symmetric about the central ray of the fan beam. However, this may interfere with the principles involved in the design of the secondary mask. In essence, the computational choices affect the design of the coherent scatter imaging system. If we anticipate a computationally efficient forward model which includes online computation of the mask modulation geometry factor, then the secondary mask should be designed to allow for some or all of the mirror symmetries. On the other hand, given a fixed secondary mask, the algorithm should be designed to accommodate the potential absence of mirror-symmetric secondary masks, as we will see later. This algorithmic choice is directly influenced by the system design.

\subsection{Scatter Angle Interpolation}
To accelerate the computation of the forward model in Eq.~\ref{eq:forward_model_energy_integrating_q_integrate}, the spectral factor $S(\theta, q)$ is precomputed by using a set of predefined uniformly sampled scatter angles. We then approximate the spectral factor at other scatter angles by interpolating.  Given that the set of momentum transfer values is also predefined, the spectral factor matrix $S(\theta, q)$ can be precomputed. To get the approximate values of $S(\theta, q)$ used in the $q$-loop, nearest neighbor interpolation over $\theta$ may be utilized. 

An example of using scatter angle interpolation to approximate the true values of $S(\theta, q)$ is shown in Fig.~\ref{fig:source_factor_matrix_slice_energy_integrating}. The curves correspond to the true and interpolated spectral factors for a filtered 125 kVp source at $q = 0.2$\si{\angstrom}$^{-1}$. 250 uniformly sampled scatter angles from 0.2 to $\pi$/6 radians are used in generating the interpolated spectral factors. From the figure, we can see that the spectral factor is smooth and slowly varying with $\theta$, except in areas surrounding the characteristic peaks of the x-ray source. This suggests that interpolation will lead to fairly accurate values of the spectral factor, for a standard polychromatic x-ray source. Since the effective spectral factor $W$ is a linear combination of spectral factors, it can also be adequately approximated by interpolation.

Note that such practical polychromatic sources introduce blurring in the resulting scatter angles, for each momentum transfer value. That is, according to Bragg's equation (see Eq.~\ref{eq:bragg}), a monochromatic source would yield a single scatter angle for each momentum transfer. However, for a polychromatic source, each momentum transfer has a range of scatter angles that result, of known intensity. For a monochromatic source, the spectral factor is no longer slowly varying, so that interpolation is no longer a valid strategy. However, for such a source, each value of $q$ will have a corresponding unique value of $\theta$ associated with it, making acceleration techniques such as interpolation, unnecessary. 

\begin{figure}%
\includegraphics[width = \linewidth]{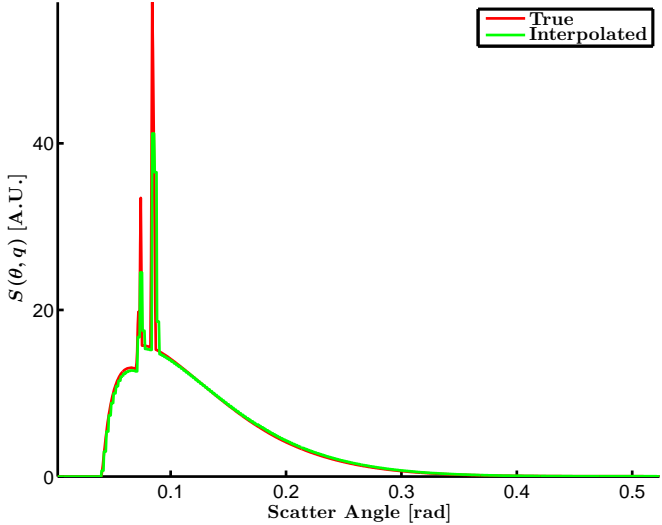}%
\caption{True and angle-interpolated spectral factors $S(\theta, q)$ that transforms object scatter density as a function of momentum transfer to object scatter density as a function of scatter angle. The spectral factors were obtained using a filtered 125 kVp source at $q = 0.2$\si{\angstrom}$^{-1}$.}%
\label{fig:source_factor_matrix_slice_energy_integrating}%
\end{figure}

Lines~\ref{lst:forw:sai_1} and \ref{lst:forw:sai_2} of Algorithm~\ref{alg:efficient_forward_model_energy_integrating} show the changes to the non-optimized forward model computation due to scatter angle interpolation. From the algorithm, we can see that the spectral factor $S(\theta, q)$ is precomputed and the effective spectral factor $W$ is interpolated. The precomputation of the spectral factor matrix and the interpolation of scatter angles avoids the re-computation of the spectral factor for each $(q, \theta)$ pair.

\subsection{Online-Offline Computations}
The choice of which factors of the forward model to compute online or offline depends on the amount of main memory available, the cost of computing the factor online, and the speed of loading the factor from main memory. If the factor is too large to fit in main memory, then part or all of the factor should be computed online. Moreover, if the cost of computing the factor is equivalent to the speed of loading the factor from main memory, then the computations should be performed online.

Part of the choice in online computations involves the mask modulation geometry factor. If this factor has desirable symmetry properties as outlined above, then many aspects of its use may be computed online. When the mask is completely determined through experimental measurements, then a lot of the potential symmetries break down. In addition, the types of masks that we have considered for the fan beam Monte Carlo study do not have these desirable symmetries. Taking these observations into account, we have chosen to precompute the binarized mask factors offline and load them into our code at run-time. To accommodate larger reconstructed object regions, larger detector arrays, and/or smaller main memory sizes, the mask factors may need to be computed online.

The source-to-object geometry factor $G_{s o}(\mathbf{r})$, can be precomputed and stored. Storing $G_{s o}(\mathbf{r})$ requires a very small amount of memory, since it is indexed only by the number of object spatial pixels. On the other hand, precomputing the object-to-detector geometry factor $G_{o d}(\mathbf{s})$ will require a much larger amount of memory for the fan beam system under consideration, since it is indexed by both the object spatial and detector pixels. An alternative to precomputing all of $G_{o d}(\mathbf{s})$ is interpolation between a subset of precomputed $G_{o d}(\mathbf{s})$ or precomputing only a fraction of the factors and computing the rest online, but we do not pursue these avenues in this paper. Note that storing the mask modulation geometry factor $T$ requires far less memory than the object-to-detector geometry factor since it is binary. Lines \ref{lst:forw:on_off_1}, \ref{lst:forw:on_off_2}, \ref{lst:forw:on_off_3}, \ref{lst:forw:on_off_4}, in Algorithm~\ref{alg:efficient_forward_model_energy_integrating} show changes to the computational structure of the forward model due to offline computations. The mask modulation geometry factor is computed in four parts corresponding to the transmission of scattered photons from either half of the object domain (left or right) to either half of the detector array (up or down), through the secondary mask. For example, $T_{UL}(\mathbf{r})$ is the mask transmission values for scatter vectors from a spatial location $\mathbf{r}$ on the left side of the object domain to the pixels in the upper half of the detector array.

\section{Computation of the Backward Model} \label{sec:computation_backward_model}
A non-optimized computational structure for the backward model is given in Algorithm~\ref{alg:raw_backward_model_energy_integrating}.

\begin{algorithm}
\caption{Computational structure for backward model with no optimization}\label{alg:raw_backward_model_energy_integrating}
\begin{algorithmic}
\State Given detector image $\mathbf{g}(\mathbf{r}^{\prime})$
\For{\textbf{each} object point $\mathbf{r}$}
	\For{\textbf{each} momentum transfer $q$}
		\For{\textbf{each} detector point $\mathbf{r}^{\prime}$}				
			\State Compute $G_{s o}(\mathbf{r})$
			\State Compute $G_{o d}(\mathbf{s})$		
			\State Compute $T(\mathbf{r}, \mathbf{r}^{\prime})$	
			\State Compute $\Delta\theta$
			\State Compute scatter angle $\theta(\mathbf{r}, \mathbf{r}^{\prime})$ 
			\State Compute $S(\theta, q)$
			\State $\begin{aligned} f(\mathbf{r}, q) &\pluseq G_{s o}(\mathbf{r}) \times G_{o d}(\mathbf{s}) \times T(\mathbf{r}, \mathbf{r}^{\prime}) \\
						&\qquad \times \Delta\theta \times S(\theta, q) \times g(\mathbf{r}^{\prime})
						\end{aligned}$	
		\EndFor	
	\EndFor
\EndFor
\State $f(\mathbf{r}, q) \mathrel{{\times}{=}} C$
\State Return object image $f(\mathbf{r}, q)$
\end{algorithmic}
\end{algorithm}

An efficient implementation of the backward model is paramount for an overall efficient iterative algorithm for object scatter density estimation. As was the case for the forward model, the geometry and spectral factors offer several opportunities for efficient computation. The scatter angle interpolation, symmetry classes, and online-offline trade-off identified for the efficient computation of the forward model can be easily incorporated into an efficient computation of the backward model. Algorithm~\ref{alg:efficient_backward_model_energy_integrating} gives such an implementation. Note that 

\begin{algorithm*}[!htp]
\caption{Computational structure for backward model incorporating the proposed optimizations. See caption of Algorithm~\ref{alg:efficient_forward_model_energy_integrating} for notation.}
\label{alg:efficient_backward_model_energy_integrating}
\begin{algorithmic}
\State \textbf{Precompute} $G_{s o}(\mathbf{r})$ for object spatial locations $\mathbf{r}$
\State \textbf{Precompute} $S(\theta, q)$ for predefined $q$ and $\theta$ samples
\State \textbf{Precompute} four symmetry-based mask factor images: $T_{UL}(\mathbf{r})$, $T_{DL}(\mathbf{r})$, $T_{UR}(\mathbf{r})$, $T_{DR}(\mathbf{r})$
\Procedure{$\mathbf{f} =$ Back\_project}{$\mathbf{g}$}%
\State \textbf{Extract} four symmetry-based detector images from $\mathbf{g}$: $\mathbf{g}_{UL}$, $\mathbf{g}_{DL}$, $\mathbf{g}_{UR}$, $\mathbf{g}_{DR}$
\For{\textbf{each} object location along $x$}
	\For{\textbf{each} object location along $y$ in the left half}		
		\If{first location along the $y$ direction}
			\State \textbf{Compute} $G_{o d}(1:M/2, 1:N)$ 
		\Else
			\State \textbf{Update} $G_{o d}(1:M/2, \rho_y+1:N) = G_{o d}(1:M/2, 1:N-\rho_y)$ 
			\State \textbf{Recompute} $G_{o d}(1:M/2, 1:\rho_y)$
		\EndIf
		\State Compute $\Delta\theta$
		\State \textbf{Fetch} $G_{s o}([x, y, 0])$ 
		\State \textbf{Fetch} $T_{UL}([x, y, 0])$, $T_{DL}([x, y, 0])$, $T_{UR}([x, -y, 0])$, and $T_{DR}([x, -y, 0])$ 		
		\State $\mathbf{g}_L = G_{s o}([x, y, 0]) \odot G_{o d} \odot \Delta\theta \odot (T_{UL}([x, y, 0]) \odot \mathbf{g}_{UL} + T_{DL}([x, y, 0]) \odot \mathbf{g}_{DL})$
		\State $\mathbf{g}_R = G_{s o}([x, y, 0]) \odot G_{o d} \odot \Delta\theta \odot (T_{UR}([x, -y, 0]) \odot \mathbf{g}_{UR} + T_{DR}([x, -y, 0]) \odot \mathbf{g}_{DR})$		
		\State Compute scatter angles $\theta(\mathbf{r}, \mathbf{r}^{\prime})$
		\For{\textbf{each} momentum transfer $q$}
			\State \textbf{Interpolate} $S(\theta, q)$ and store in $S^{\text{interp}}$					
			\State $f([x, y, 0], q) = C \times \left(\mathbf{g}_L \bullet S^{\text{interp}}\right)$
			\State $f([x, -y, 0], q) = C \times \left(\mathbf{g}_R \bullet S^{\text{interp}}\right)$
		\EndFor
	\EndFor
\EndFor
\EndProcedure
\end{algorithmic}
\end{algorithm*}

\noindent unlike in Algorithm~\ref{alg:efficient_forward_model_energy_integrating} where the scatter angle interpolation was performed on the effective spectral factor $W$ after integrating out the $q$ dimension, the scatter angle interpolation of the spectral factor $S$ in Algorithm~\ref{alg:efficient_backward_model_energy_integrating} occurs independently for each $q$. This results in the backward model being slower than the forward model. 

In the next section, we setup the optimization problem in which the forward and backward models play a crucial role in the recovery of the underlying image.

\section{Image Recovery Problem Description} \label{sec:image_recovery_problem_desc}
The measurements $y_i$ from an x-ray coherent scatter imaging system are modeled as independent Poisson distributed random variables

\begin{equation}
\label{eq:poisson_model}
y_i \sim \mbox{Poisson}(\sum_{j = 1}^J H_{i, j} f_j + r_i), i = 1, \dots, I
\end{equation}
\noindent{where $I$ is the number of measurements, $J$ is the number of image voxels. $H \in \mathbb{R}_+^{I \times J}$ is the system matrix (discretized forward model) with $H_{i, j}$ denoting the $ij$\ts{th} entry. The column vector $\mathbf{f} \in \mathbb{R}_+^J$ is the lexicographic ordering of the hyperspectral image to be recovered, with $f_j$ denoting the $j$\ts{th} entry. $\mathbf{r} \in \mathbb{R}_+^I$ are the known background measurements, with the $i$\ts{th} entry denoted by $r_i$.}  Let $J = B \times Q$, where $B$ is the number of spatial bins in the image and $Q$ is the number of spectral (momentum transfer) channels.

We consider a penalized Poisson negative log-likelihood function of the form 

\begin{equation}
\label{eq:overall_objective}
J(\mathbf{f}) = L(\mathbf{f}) + \beta R(\mathbf{f}),
\end{equation}
\noindent{where $R(\mathbf{f})$ is the regularizer, $\beta > 0$ is the regularization coefficient, and}

\begin{equation*}
L(\mathbf{f}) = \sum_{i = 1}^{I} d_i(l_i)
\end{equation*}
\noindent{is the negative log-likelihood function, with}

\begin{equation*}
d_i(l) = \left(l + r_i\right) - y_i \ln\left(l + r_i\right) + \ln\left(y_i!\right)
\end{equation*}
\noindent{and $l_i = \sum_{j = 1}^J H_{i, j} f_j$, $i = 1, \dots, I$.}

A standard edge-preserving regularizer, with independent spectral bins, is chosen for $R(\mathbf{f})$ and is given by 
\begin{equation}
\label{eq:standard_EPR}
R(\mathbf{f}) = \sum_{j = 1}^{J} \sum_{k \in \mathcal{N}_j} w_{j, k} \psi_{\delta}\left(f_j - f_k\right),
\end{equation}
\noindent{where $\psi_{\delta}\left(\cdot\right)$ is the edge-preserving potential function, with scale parameter $\delta$, which is symmetric, convex, and possesses desirable smoothness properties \cite{Huber1981, Lange1990, Green1990}, $\mathcal{N}_j$ is the set of neighbors of the $j$\ts{th} image voxel, and $w_{j, k}$ is a neighborhood weight to compensate for different physical units of the spectral and spatial dimensions, and different voxel sizes in each dimension.} We assume a spatially piecewise smooth object. In general, the momentum transfer profile (MTP) of a material may not be piecewise smooth, so that only spatial neighbors are permitted in $\mathcal{N}_j$ for each image voxel.

The constrained convex optimization problem of interest is then

\begin{equation}
\label{op:orig_problem}
\begin{aligned}
& \minimize_{\mathbf{f}} & & J(\mathbf{f}) \\
& \text{subject to} & & \mathbf{f} \geq \mathbf{0}.
\end{aligned}
\end{equation}

\section{Reconstruction Algorithm} \label{sec:reconstruction_algorithm}
To solve optimization problem \ref{op:orig_problem}, we consider a sequence of simpler optimization problems obtained by lifting the objective function around the previous image estimate. In particular, we obtain a surrogate objective function which is fully separable with respect to the image parameters. This choice allows us to utilize the optimized code for online forward and backward models computation. Algorithm~\ref{alg:image_update_EM_type} shows the steps involved in solving the optimization problem in \eqref{op:orig_problem} using the sequence of convex optimization problems given in \eqref{op:modified_problems} as detailed in Appendix~\ref{app:reconstruction_algorithm_details}.

\begin{algorithm}[!ht]
\caption{EM-type image reconstruction algorithm. $\textproc{Forward\_project}(\cdot)$ and $\textproc{Back\_project}(\cdot)$ are obtained from Algorithms~\ref{alg:efficient_forward_model_energy_integrating} and \ref{alg:efficient_backward_model_energy_integrating}, respectively. $\boldsymbol{1}$ is a vector of ones. $\oslash$ denotes element-wise division.}
\label{alg:image_update_EM_type}
\begin{algorithmic}
\State Given the measured data $\mathbf{y}$ and background $\mathbf{r}$
\State Initialize the neighborhood structure $\boldsymbol{\mathcal{N}} = \left\{\mathcal{N}_j\right\}$, $\mathbf{w} = \left\{w_{j, k}\right\}$ 
\State Initialize the image $\mathbf{f}^{0}$
\State Precompute $\mathbf{b}^{(1)} = \textproc{Back\_project}(\boldsymbol{1})$ 
\For{$t=0$ \textbf{to} $T-1$}
	\State Compute $\mathbf{z} = \textproc{Forward\_project}(\mathbf{f}^t)  + \mathbf{r}$ 
	\State Compute $\mathbf{b}^{(2)} = \textproc{Back\_project}(\mathbf{y} \oslash \mathbf{z})$ 
	\State Compute $\mathbf{f}^{t + 1}$ using Eq.~\ref{eq:f_update}
\EndFor
\end{algorithmic}
\end{algorithm}

To accelerate the convergence rate of the EM-type algorithm, we employ ordered subsets, a range decomposition (measurement space partition) method. Algorithm~\ref{alg:image_update_OSEM_type} shows the structure of the ordered subset EM-type algorithm. There are several choices for the measurement space partition. One particular choice preserves the symmetry classes identified in Section~\ref{sec:computation_forward_model}. To preserve left-right mirror symmetry, when a detector pixel belongs to a subset, the pixel corresponding to its mirror reflection about the line $y = 0$ should belong to the same subset. Moreover, to preserve up-down mirror symmetry, when a detector pixel belongs to a subset, the pixel corresponding to its mirror reflection about the line $z = 0$ should belong to the subset. In order to satisfy both left-right and up-down mirror symmetries, if one detector pixel is in a subset, then its three mirror reflections must also be included in that subset. In addition, to utilize translation symmetry to reduce computation, pixels that are $\rho_y$ pixels away (along the $y$ direction) from each of the four symmetric pixels (original pixel + 3 reflection pixels) must also be included in the subset. 

If we consider the horizontal and vertical indexes of the detector pixels, the vertical indexes can be used to satisfy up-down symmetry constraints while the horizontal indexes can be used to address left-right and translation symmetry concerns. Due to up-down mirror symmetry, we can consider partitioning half (up or down) of the vertical detector pixel indexes. The partitioning of the other half follows directly from up-down mirror reflections. For a detector array with $M$ rows, we can partition either the top or bottom $M/2$ rows using any strategy. We have chosen to select members of a subset by using a fixed step size $\rho_z$. To balance the vertical subsets, $\rho_z$ should be a factor of $M/2$. The detector pixels with vertical indexes $m: \rho_z: M/2$ and $M-m+1: -\rho_z: M/2+1$ belong to the same subset, where $1\leq m\leq \rho_z$ is the smallest vertical index in the subset. Note that we have used MATLAB's listing notation $a: b: c$ to mean numbers from $a$ to $c$ in steps of $b$, $a$ inclusive. 

In order to satisfy left-right symmetry constraints, the same strategy used for the vertical indexes can be used for the horizontal indexes. However, to also satisfy translation symmetry constraints, extra precaution must be taken. The difference between the horizontal indexes of detector pixels that belong to the same subset must be a multiple of $\rho_y$. For a detector array with $N$ columns, satisfying translation symmetry implies that the detector pixels with horizontal indexes $n: \rho_y: N$ belong to the same subset, where $1\leq n\leq \rho_y/2$ is the smallest horizontal index in the subset. In addition, to satisfy left-right mirror symmetry, the detector pixels with horizontal indexes $N-n+1: -\rho_y: 1$ also belong to that subset. To balance the horizontal subsets, $\rho_y$ should be a factor of $N$ and even.

\begin{figure*}[!ht]%
\includegraphics[width = \linewidth]{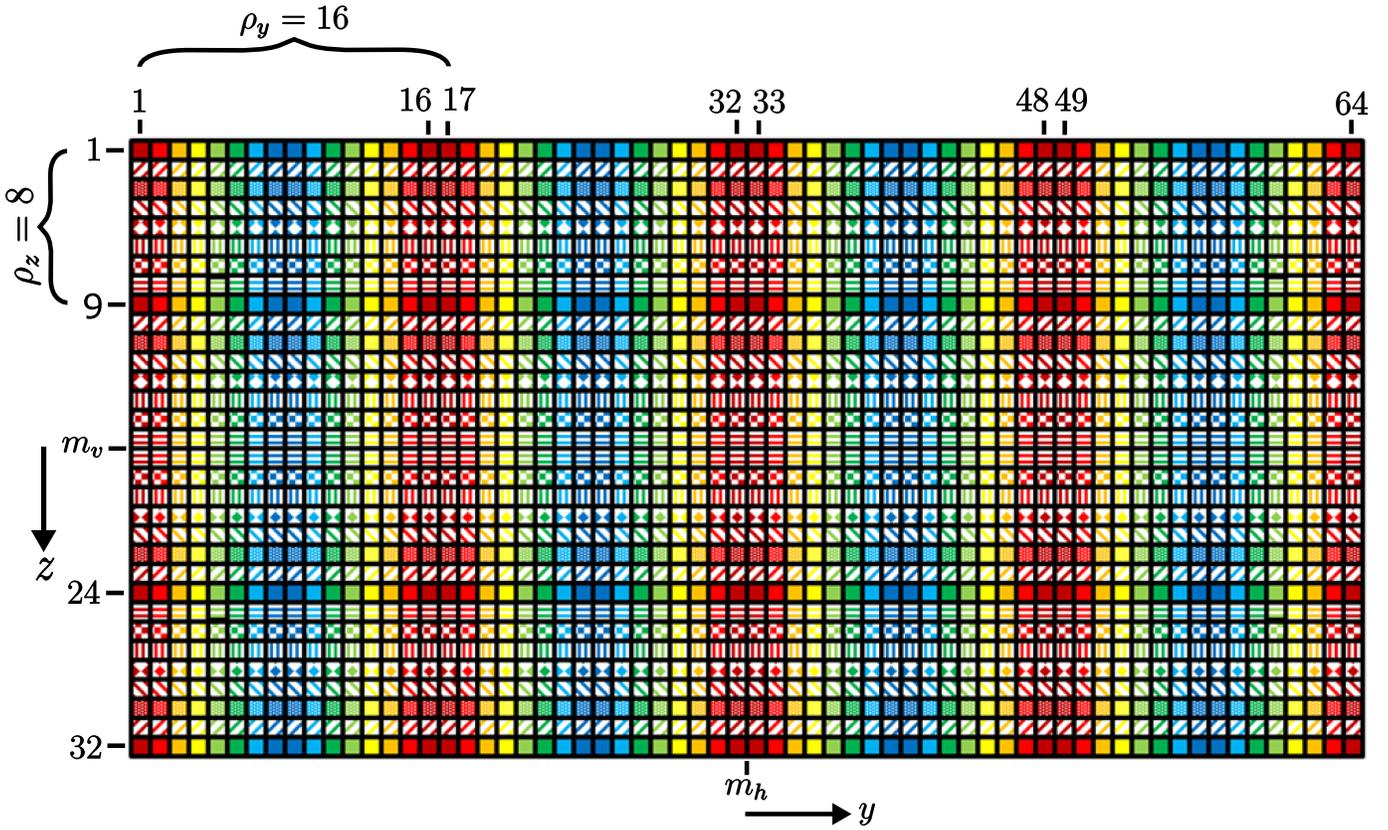}%
\caption{Layout of the symmetry-preserving choice of 64 subsets based on a mini detector array with 32 rows and 64 columns. The translation symmetry step size is $\rho_y = 16$. $m_h$ and $m_v$ represent the horizontal and vertical midpoints of the detector array and correspond to the lines $z = 0$ and $y = 0$, respectively. Pixels with the same combination of color and pattern belong to the same subset.}%
\label{fig:OS_choice}%
\end{figure*}

\begin{algorithm}[!ht]
\caption{Ordered subset expectation-maximization type image reconstruction algorithm. $\textproc{Forward\_project}_p(\cdot)$ and $\textproc{Back\_project}_p(\cdot)$ are obtained from Algorithms~\ref{alg:efficient_forward_model_energy_integrating} and  \ref{alg:efficient_backward_model_energy_integrating}, respectively, by utilizing the appropriate partition of the measurement (detector) space. $\boldsymbol{1}$ is a vector of ones. $\oslash$ denotes element-wise division.}
\label{alg:image_update_OSEM_type}
\begin{algorithmic}
\State Given the measured data $\mathbf{y}$ and background $\mathbf{r}$
\State Partition $\mathbf{y}$ and $\mathbf{r}$ into $P$ disjoint subsets $\left\{\mathbf{y}_p\right\}$ and $\left\{\mathbf{r}_p\right\}$
\State Initialize the neighborhood structure $\boldsymbol{\mathcal{N}} = \left\{\mathcal{N}_j\right\}$, $\mathbf{w} = \left\{w_{j, k}\right\}$ 
\State Initialize the image $\mathbf{f}^{0}$
\State Precompute $\mathbf{b}^{(1)}_p = \textproc{Back\_project}_p(\boldsymbol{1})$, $\forall p = 1, \ldots, P$
\For{$t=0$ \textbf{to} $T-1$}
	\State $\mathbf{f}^{t, 0} = \mathbf{f}^{t}$
	\For{$p=1$ \textbf{to} $P$}
		\State Compute $\mathbf{z}_p = \textproc{Forward\_project}_p(\mathbf{f}^{t, p-1})  + \mathbf{r}_p$ 
		\State Compute $\mathbf{b}^{(2)}_p = \textproc{Back\_project}_p(\mathbf{y}_p \oslash \mathbf{z}_p)$
		\State Compute $\mathbf{f}^{t, p}$ using Eq.~\ref{eq:f_update}
	\EndFor
	\State $\mathbf{f}^{t + 1} = \mathbf{f}^{t, P}$
\EndFor
\end{algorithmic}
\end{algorithm} 

For illustration, we consider the partitioning of a scaled-down version of the detector array with 32 rows and 64 columns. The partitions shown in Fig.~\ref{fig:OS_choice} satisfy the requirements for preserving symmetries, when the translation symmetry step size is $\rho_y = 16$ and the fixed step size in the vertical direction is 8 pixels. In the figure, $m_h$ and $m_v$ represent the horizontal and vertical midpoints of the detector array, corresponding to the lines $z = 0$ and $y = 0$, respectively. Pixels with the same combination of color and pattern belong to the same subset. Using a fixed step size of 8 pixels gives 8 different subsets along the vertical direction. The choice of 16 as the translation symmetry step size gives a total of 8 (= 16/2) horizontal subsets. This gives a total of 64 subsets as shown in Fig.~\ref{fig:OS_choice}. For example, in each row, the pixels with horizontal indexes (1, 17, 33, 49) and (64, 48, 32, 16) belong to the same subset. Also, in each column, the pixels with vertical indexes (1, 9) and (32, 24) belong to the same subset. The step sizes that were used for partitioning the mini detector array were applied to the full detector array used in the simulations that follow.

\section{Results} \label{sec:results}
The forward and backward models, integral parts of the OSEM-type iterative algorithm for estimating object scatter density (see Algorithm~\ref{alg:image_update_OSEM_type}), were validated on simulated analytic and Monte Carlo data. The subsets were chosen in the same way as the illustrative example in the previous section. For both simulations, the source was located at the origin and the center of the flat-panel energy-integrating detector array was 1546.5 mm away along the positive x-axis. The source was operated at 125 kVp, and the spectrum was filtered by 0.5 mm of aluminum, before being shaped into a fan by the primary aperture (slit). The secondary mask was placed 100 mm in front of, and parallel to, the plane of the detector array. The detector array had 1536 rows and 2048 columns, with a pixel pitch of 0.19 mm in both $z$ and $y$ directions. The center of the reconstructed object was located 1035 mm from the source, along the positive x-axis. A region of 70 mm by 85.12 mm, in the $xy$-plane, was reconstructed, with a pixel pitch of $\Delta x = 2.5$ mm and $\Delta y = 3.04=(16 \times 0.19)$ mm. 79 evenly spaced momentum transfer bins from 0.01 to 0.4 \si{\angstrom}$^{-1}$ were used. There were $N_{\theta}$ scatter angle samples chosen uniformly from 0 to $\pi/6$ radians, excluding 0. The secondary mask shown in Fig.~\ref{fig:secondary_mask_MURI} was used in both the analytic forward model and Monte Carlo simulated data. Since the mask is not amenable to the mirror symmetry classes, the four symmetric mask modulation factors are precomputed, as stated in Algorithm~\ref{alg:efficient_forward_model_energy_integrating}.

\subsection{Analytic Simulation} \label{sec:analytic_results}
The analytic forward model was used to model an object containing two vials of strong scatterers in close proximity, illustrated in Fig.~\ref{fig:simulated_spatial_dist_vial_configuration}. Rectangular vials of sodium chloride (NaCl) and aluminum (Al) crystalline powder were placed along the direction of the fan beam's central ray. Each vial occupies a rectangular region 10 mm by 12.16 mm in size. The vials are separated by 20 mm along the positive x-axis.

The object scatter density was simulated by inserting each material's momentum transfer profile at the appropriate spatial location. The momentum transfer profile of each material was obtained experimentally by using an x-ray diffractometer \cite{Bruker} and then interpolated to match the sampling of the momentum transfer space. The resulting momentum transfer profiles (MTPs) of NaCl and Al are shown as the red reference curves in Figs~\ref{fig:NaCl_profile_vial_configuration} and \ref{fig:Al_profile_vial_configuration}, respectively. The spatial distribution of the simulated object ($\tilde{f}(x, y) = \int f(x, y, q) \, dq$), obtained by summing over all the momentum transfer channels of the object scatter density, is shown in Fig.~\ref{fig:simulated_spatial_dist_vial_configuration}. The non-optimized forward model shown in Algorithm~\ref{alg:raw_forward_model_energy_integrating} was applied to the object scatter density in creating the noise-less scatter data, to prevent the simulated data from being corrupted by scatter angle interpolation, which is utilized by the optimized forward model shown in Algorithm~\ref{alg:efficient_forward_model_energy_integrating}. Poisson noise was later introduced with a maximum photon count of 50, across all detector pixels. The simulated noisy data is given in Fig.~\ref{fig:simulated_data_vial_configuration}.

Next, we characterize the computational-time speedup of the forward and backward models introduced by each element of the proposed optimization described in Section~\ref{sec:computation_forward_model}, relative to the non-optimized models in Algorithms~\ref{alg:raw_forward_model_energy_integrating} and \ref{alg:raw_backward_model_energy_integrating}. The analytic simulated data and object was used for this characterization.

To characterize the speedup in computational time due to each element of the proposed optimization discussed in Section~\ref{sec:computation_forward_model}, the variants of the forward model were applied to the simulated vial object, whereas those of the backward model were applied to the noise-less simulated data. Both full-data and ordered-subsets implementations of the forward and backward models were tested. For the ordered-subset implementations, the models were iterated over all the subsets. The forward and backward models were implemented in MATLAB\textsuperscript{\textregistered} R2015a, on a dual processor, 6 cores per processor, 128 GB Windows machine. Table~\ref{tab:computational_time_speedups} shows the time taken to apply each model once on the appropriate data and the speedup in computational time relative to the time taken by the non-optimized model. The speedup is computed as the ratio of the computational time of an algorithm to the computational time of the non-optimized version. The table also shows the error associated with each element of optimization relative to the non-optimized version. The results for the ordered-subsets implementations are given to the right of the results for the full-data implementations. For the models that utilize scatter angle interpolation, the number of scatter angle bins used in interpolation is $N_{\theta} = 250$.

\definecolor {gray1}{gray}{0.8}
\definecolor {gray2}{gray}{1}

\begin{table*}[!ht]
\setlength\arrayrulewidth{1pt}
\centering
\caption{Computational time, speedup, and relative error for different levels of optimization of the forward and backward models. Speedup is defined as the ratio of the computational time of an algorithm to the computational time of the non-optimized version. MATLAB\textsuperscript{\textregistered}'s tic-toc functions were used in timing the variants of the forward and backward models. The number of scatter angle bins used in interpolation is $N_{\theta} = 250$. The NRMSE is computed relative to the results (in the measurement or object space) of the non-optimized model. To obtain the NRMSE, the RMSE is normalized by the square root of the mean of the square of the entries of the result obtained using the non-optimized model. The results for the ordered-subsets implementations are given to the right of the results for the full-data implementations. Note that the cumulative effect of all the optimizations on the speedup is not linear.}
\label{tab:computational_time_speedups}
\begin{tabular}{p{1.5cm}||cc||cc||cc||cc||cc||cc}
\toprule 
\multirow{3}{*}{Optimization} & \multicolumn{12}{c}{Model}\tabularnewline
 & \multicolumn{6}{c}{Forward} & \multicolumn{6}{c}{Backward}\tabularnewline
 & \multicolumn{2}{c}{Time (s)} & \multicolumn{2}{c}{Speedup} & \multicolumn{2}{c}{NRMSE (\%)} & \multicolumn{2}{c}{Time (s)} & \multicolumn{2}{c}{Speedup} & \multicolumn{2}{c}{NRMSE (\%)} \tabularnewline
 & FDI & OSI & FDI & OSI & FDI & OSI & FDI & OSI & FDI & OSI & FDI & OSI \tabularnewline
\midrule 
\midrule 
\rowcolor{gray1}NO  & 10516.88 & 9466.80 & \textbf{1.00} & \textbf{1.00} & 0.00 & 0.00 & 10512.32 & 9365.06 & \textbf{1.00} & \textbf{1.00} & 0.00 & 0.00 \tabularnewline
\midrule[\heavyrulewidth]
\rowcolor{gray2}SAI & 319.45 & 240.24 & 32.92 & 39.41 & 6.20 & 6.20 & 1191.05 & 1330.34 & 8.83 & 7.04 & 0.67 & 0.67 \tabularnewline
\rowcolor{gray1}TS   & 10385.48& 9434.27 & 1.01 & 1.00 & 0.00 & 0.00 & 10438.93 & 9472.91 & 1.01 & 0.99 & 0.00 & 0.00 \tabularnewline
\rowcolor{gray2}LRMS & 5271.56 & 4822.61 & 2.00 & 1.96 & 0.00 & 0.00 & 5317.47 & 4763.28  & 1.98 & 1.97 & 0.00 & 0.00 \tabularnewline
\rowcolor{gray1}UDMS & 4993.05 & 4860.33 & 2.11 & 1.95 & 0.00 & 0.00 & 4975.95 & 4791.84  & 2.11 & 1.95 & 0.00 & 0.00 \tabularnewline
\rowcolor{gray2}OOT  & 10353.90 & 9525.87 & 1.02 & 0.99 & 0.00 & 0.00 & 10315.77 & 9441.22 & 1.02 & 0.99 & 0.00 & 0.00 \tabularnewline
\rowcolor{gray1}AO  & 61.04 & 64.70 & \textbf{172.30} & \textbf{146.32} & 6.20 & 6.20 & 287.40 & 288.98 & \textbf{36.58} & \textbf{32.41} & 0.67 & 0.66 \tabularnewline
\bottomrule
\end{tabular}
\vspace{4 mm}

\raggedright
NO = No Optimization; SAI = Scatter Angle Interpolation; TS = Translation Symmetry; LRMS = Left-Right Mirror Symmetry; UDMS = Up-Down Mirror Symmetry; OOT = Online-Offline Trade-off; AO = All Optimizations; NRMSE = Normalized RMSE; RMSE = Root Mean Square Error; FDI = Full-Data Implementation; OSI = Ordered-Subsets Implementation
\end{table*}


From the table, we see that the greatest improvement in computational time is due to scatter angle interpolation; it gives an improvement of a factor of about 33 (39 for OS) and 9 (7 for OS) in the application of the forward and backward model, respectively, relative to the non-optimized models. The next largest improvement in computational time is due to either left-right or up-down mirror symmetry, which give an improvement of a factor of about 2 in both models. The improvement due to translation symmetry and online-offline trade-off are marginal at best. The forward and backward models are accelerated by a factor of approximately 172 (146 for OS) and 37 (32 for OS), respectively, by utilizing all the elements of our optimized algorithms. Note the non-linearity in speedup of the cumulative effect of the elements of our proposed optimization. The scatter angle interpolation introduces errors of approximately 6\% and 1\% in the computation of the forward and backward models, respectively, relative to the non-optimized algorithms. These relative errors can be reduced by increasing the number of scatter angle samples used for the interpolation or using a higher order interpolation. However, the error reduction comes at the cost of increased computational times. When $N_{\theta} = 2000$, the times to compute the full forward and backward models are 65.28 s (78.40 s for OS) and 289.62 s (291.74 s for OS) respectively.

The ordered-subsets implementations of the fully-optimized (AO) and non-optimized (NO) forward and backward models were used in recovering the simulated vial object based on the OSEM-type algorithm in Algorithm~\ref{alg:image_update_OSEM_type}. Fig.~\ref{fig:reconstruction_results_vial_configuration} shows the results of estimating the object scattering density using 20 iterations of the OSEM-type algorithm. Each mean momentum transfer profile was obtained by averaging the MTPs across the known location occupied by a material. From the figures, we can see that the spatial distribution given in Fig.~\ref{fig:estimated_spatial_dist_vial_configuration} and the momentum transfer profiles are recovered fairly accurately in a few iterations. There is a slight shift in the peak of the MTP recovered using the fully-optimized models with 250 scatter angle bins (SABs) relative to those obtained using the non-optimized models, especially at lower momentum transfer values, corresponding to the approximation error incurred due to scatter angle interpolation. To diminish the shift, a finer sampling of the scatter angles may be used. From the figure, we see that the reconstruction results based on the fully-optimized models with 2000 SABs are a better match to those of the non-optimized models. The estimated mean measurements were obtained by applying the fully-optimized forward operator, with 250 SABs, on the estimated object scatter density. The estimated mean measurements are in agreement with the simulated noisy Poisson measurements.

\begin{figure*}[!htb]%
	\centering
	\subfigure[Simulated noisy measurements]{%
		\includegraphics[width=0.4\textwidth]{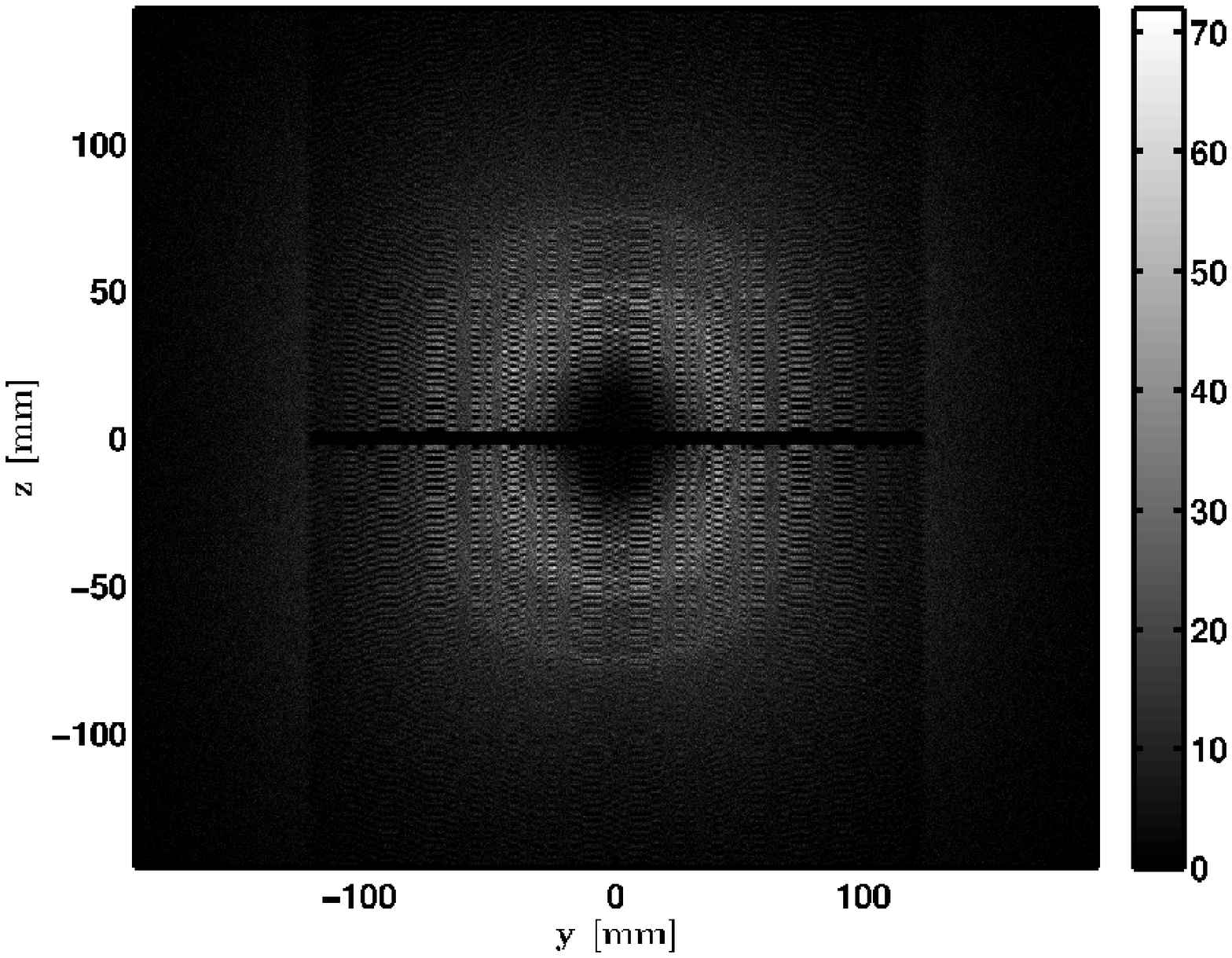}%
		\label{fig:simulated_data_vial_configuration}}%
	~
	\subfigure[Estimated mean measurements]{%
		\includegraphics[width=0.4\textwidth]{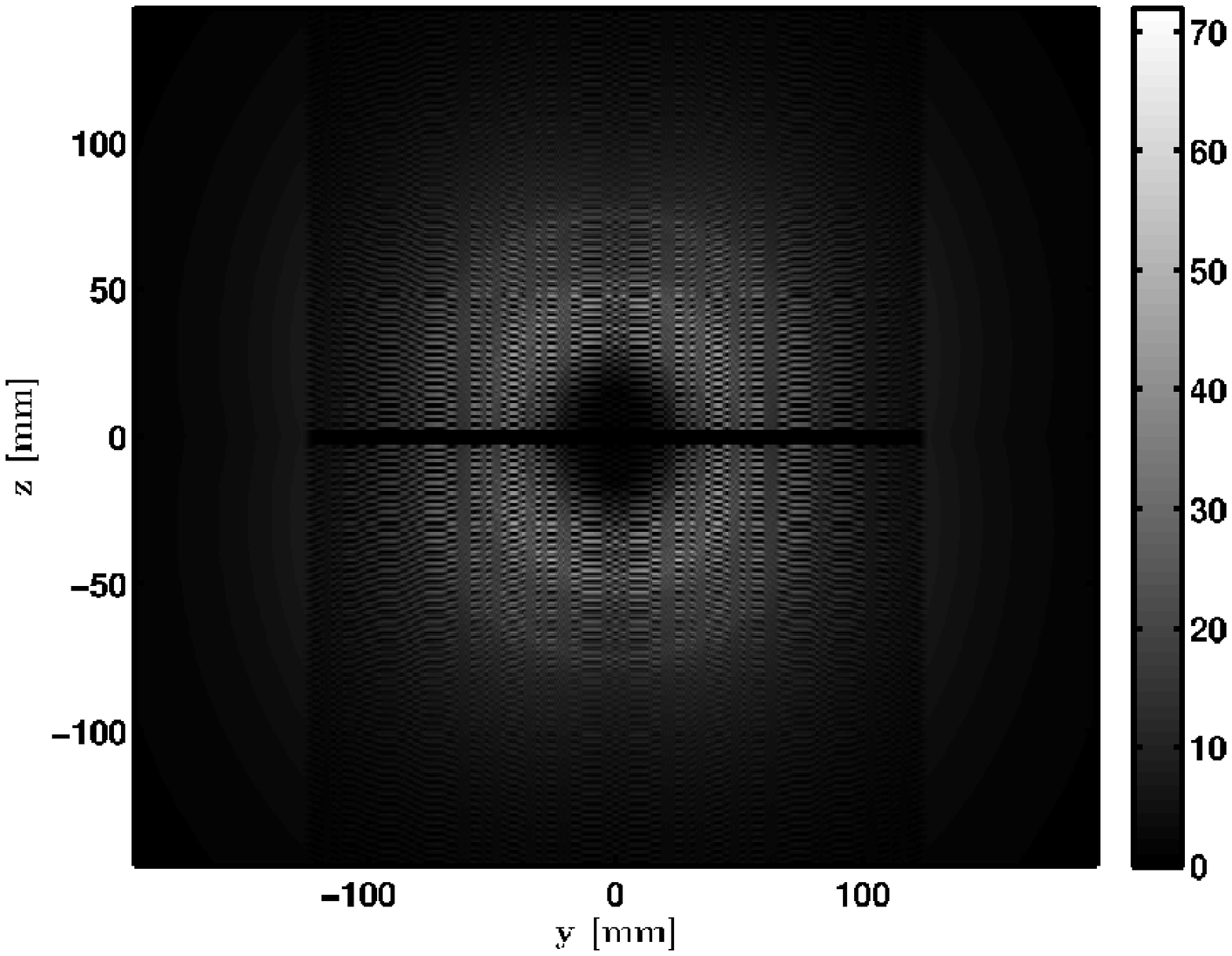}%
		\label{fig:estimated_data_vial_configuration}}%
	\\
	\subfigure[Simulated spatial distribution]{%
		\includegraphics[width=0.4\textwidth]{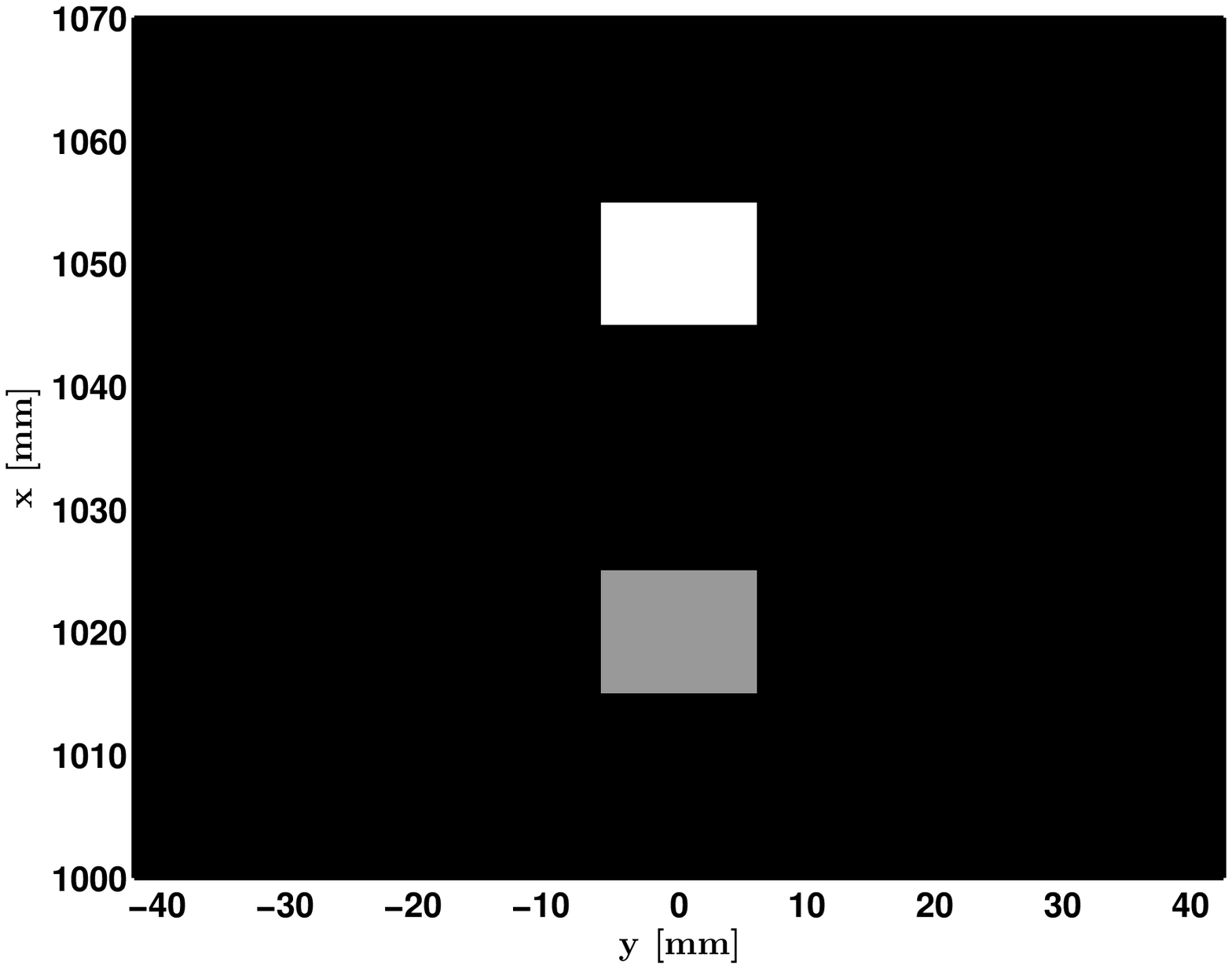}%
		\label{fig:simulated_spatial_dist_vial_configuration}}%
	~
	\subfigure[Estimated spatial distribution]{%
		\includegraphics[width=0.4\textwidth]{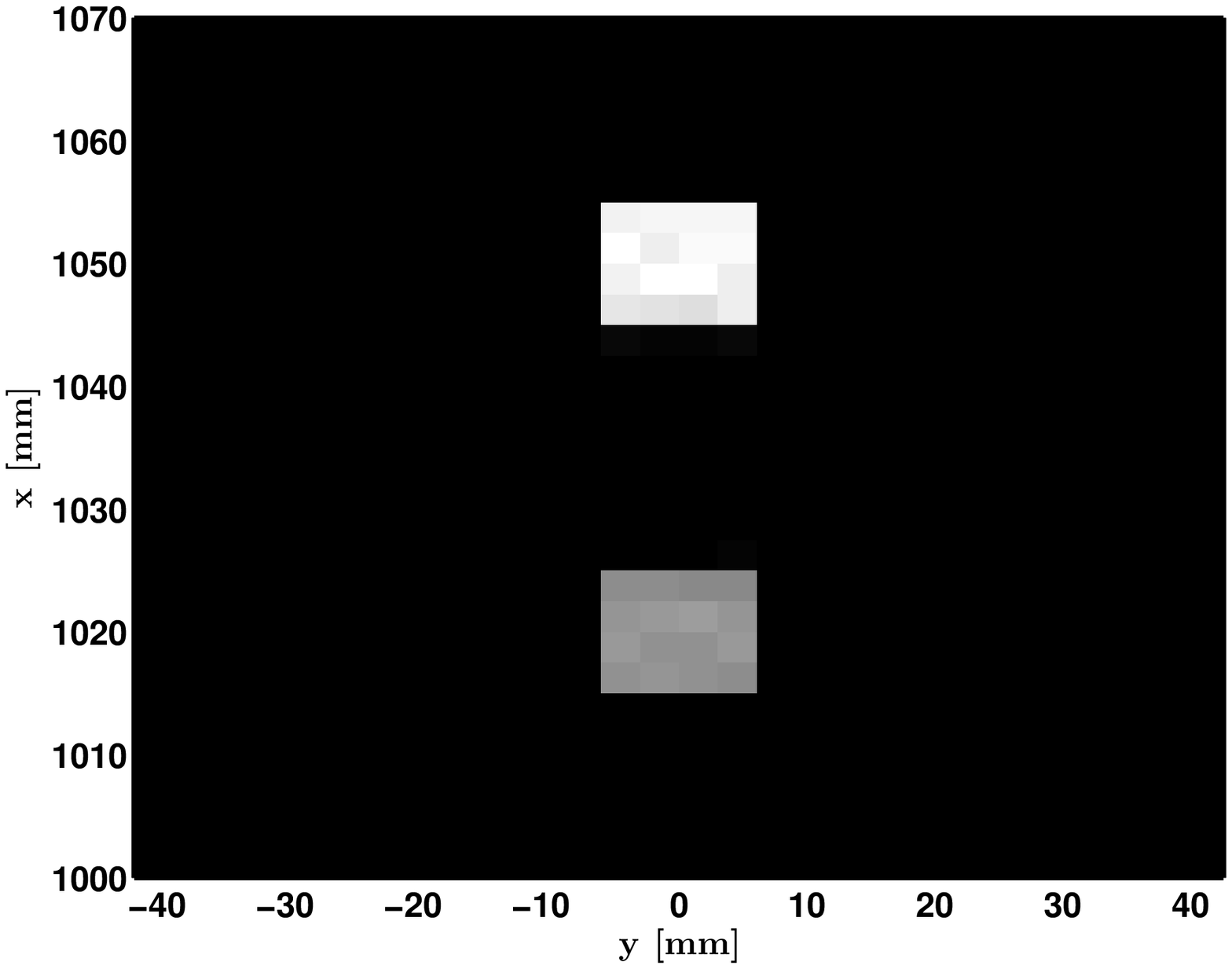}%
		\label{fig:estimated_spatial_dist_vial_configuration}}%
	\\
	\subfigure[Normalized MTP for NaCl]{%
		\includegraphics[width=0.4\textwidth]{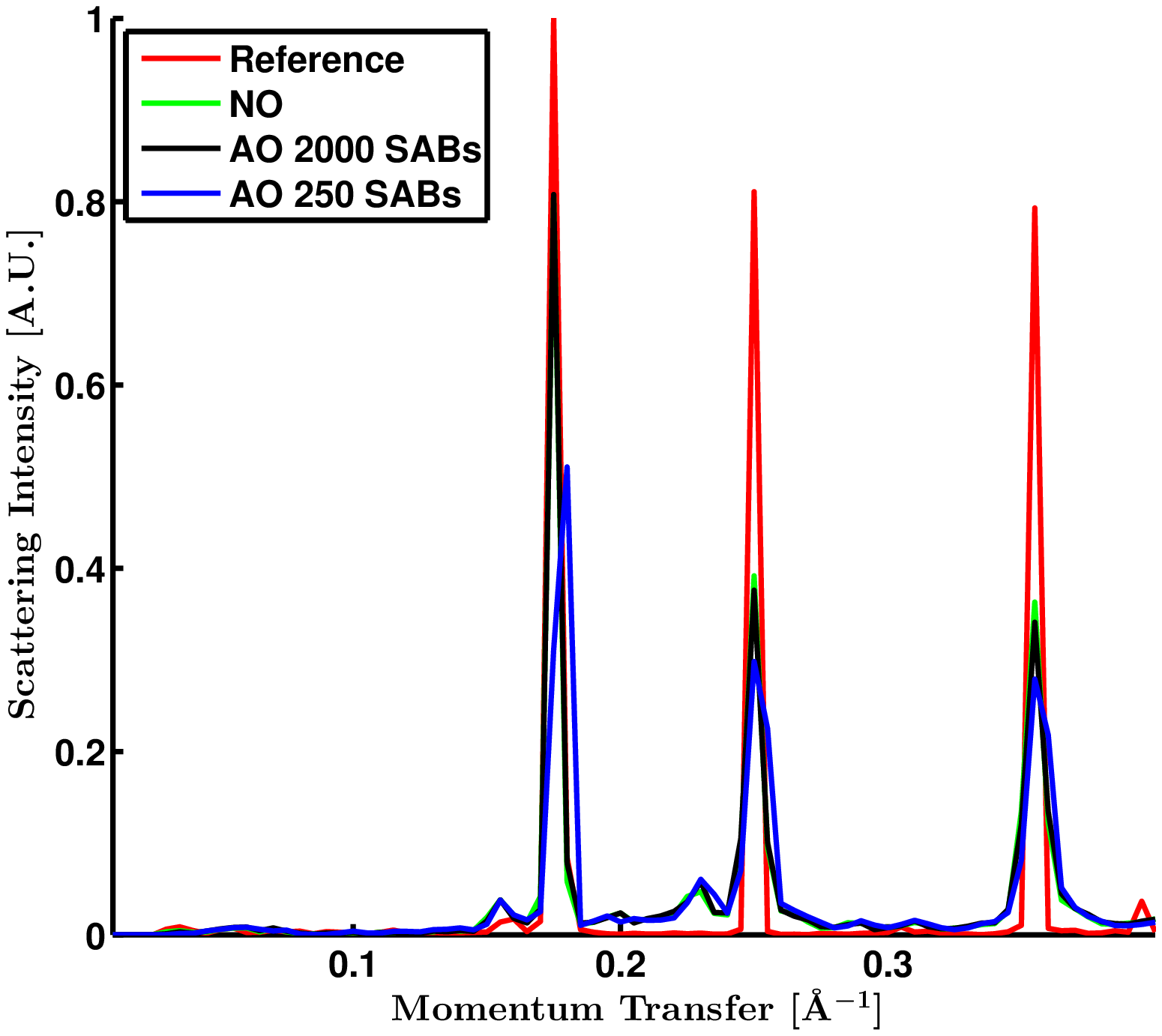}%
		\label{fig:NaCl_profile_vial_configuration}}%
	~
	\subfigure[Normalized MTP for Al]{%
		\includegraphics[width=0.4\textwidth]{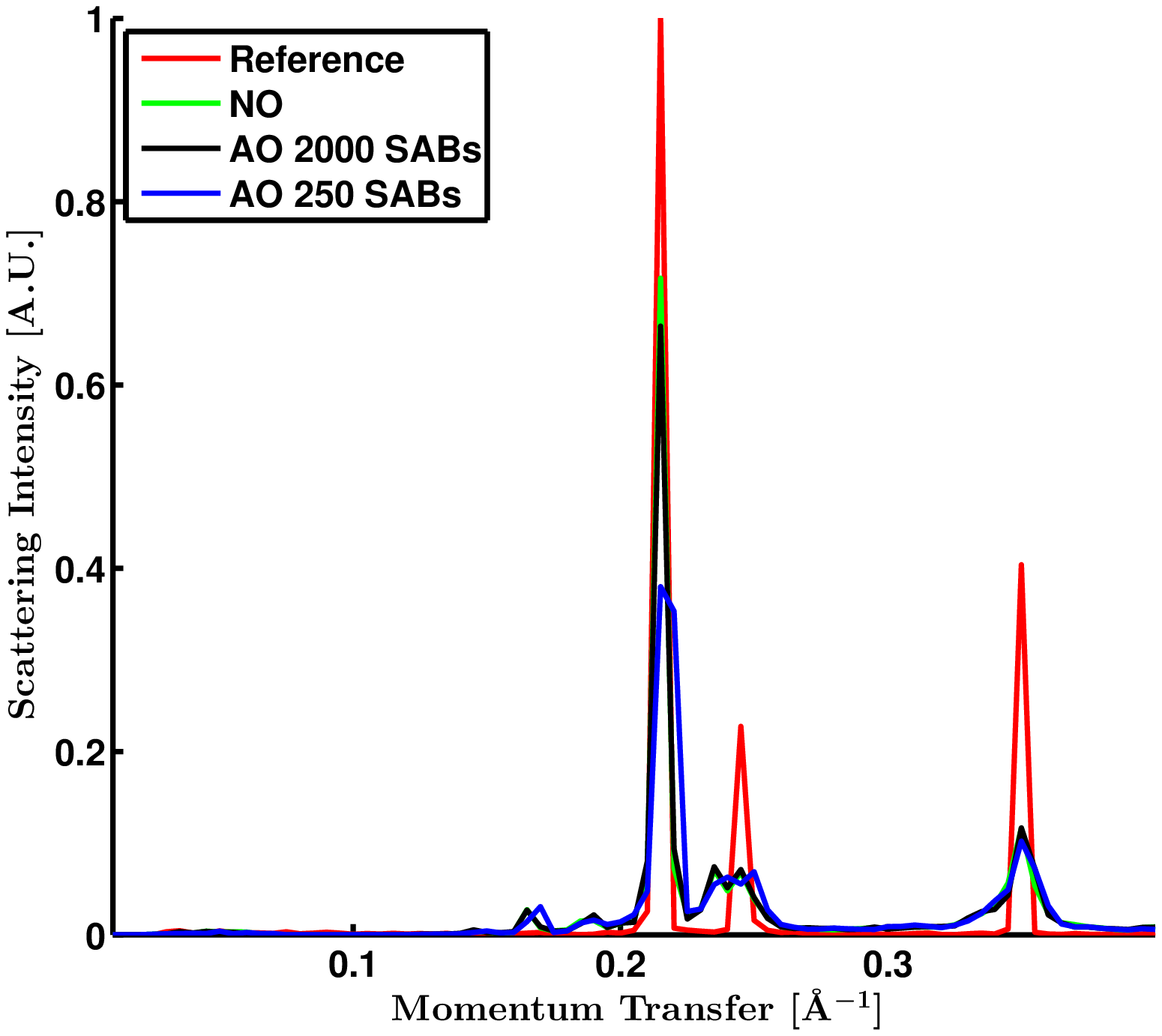}%
		\label{fig:Al_profile_vial_configuration}}%
	\caption{Estimation of object scatter density from vials of NaCl and Al crystalline powders using ordered-subsets implementations in Algorithm~\ref{alg:image_update_OSEM_type}. Each momentum transfer profile (MTP) was obtained by averaging the MTPs across each known location occupied by a material. The estimated mean measurements were obtained by applying the fully-optimized forward operator, with 250 scatter angle bins (SABs), on the estimated object scatter density. NO and AO stand for no optimization and all optimizations, respectively. Note that both NO and AO implementations use ordered subsets.}%
	\label{fig:reconstruction_results_vial_configuration}%
\end{figure*}

The estimated spatial distributions and momentum transfer profiles closely match those of the reference configuration and materials. However, from the recovered momentum transfer profile of NaCl and Al shown in Figs~\ref{fig:NaCl_profile_vial_configuration} and \ref{fig:Al_profile_vial_configuration}, respectively, we can see that a few iterations has only accurately recovered the largest peak. 

Ordered subset promises an acceleration of the convergence rate of a regular EM-type algorithm that is comparable to the number of subsets \cite{Erdogan1999}. Fig.~\ref{fig:compare_OSEM_and_EM} shows the value of the objective as a function of the iteration number. We can see that an acceleration factor of about 76 is obtained by using the 64 subsets described in Section~\ref{sec:reconstruction_algorithm}.

\begin{figure}[H]%
	\centering
	\subfigure[Initial iterations]{%
		\includegraphics[width=0.4\textwidth]{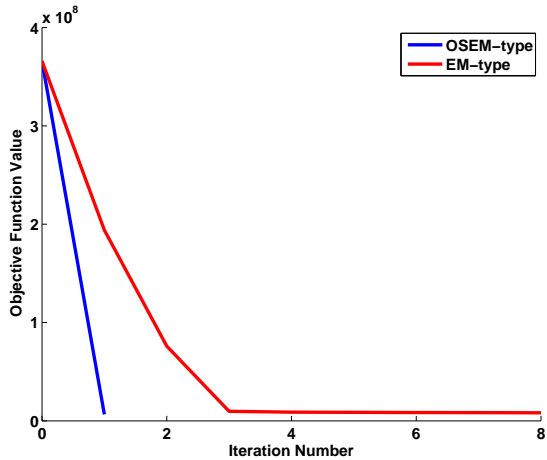}%
		\label{fig:compare_OSEM_and_EM_top}}%
	\\
	\subfigure[Subsequent iterations]{%
		\includegraphics[width=0.4\textwidth]{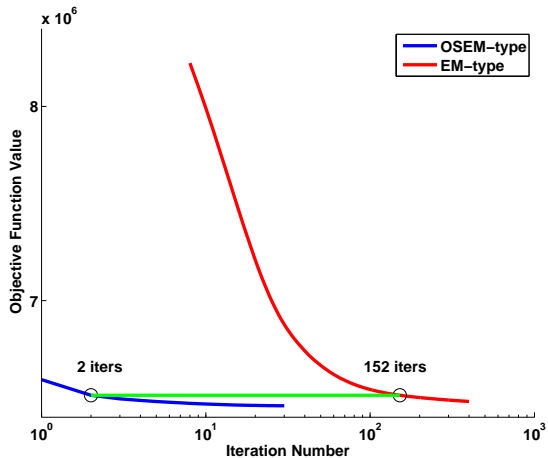}%
		\label{fig:compare_OSEM_and_EM_bottom}}%
	\caption{Objective function value versus iteration number for regular and ordered subset EM-type algorithms. The optimization curves were split into two regimes to illustrate the merits of ordered subsets. The fully-optimized operators, with 250 SABs, were utilized.}%
	\label{fig:compare_OSEM_and_EM}%
\end{figure}	

\subsection{Monte Carlo Simulation} \label{sec:MC_results}
The object utilized in the Monte Carlo simulation was a 5 mm by 50 mm by 50 mm rectangular slab of graphite powder whose momentum transfer profile is shown as the reference in Fig.~\ref{fig:mtp_profile_MC_graphite}. The Monte Carlo detector measurements includes single coherent scattering events from the slab of graphite crystalline powder, multiple scattering events, Compton scattering events, and scatter from the secondary aperture. Our forward model only accounts for single coherent scattering from the object, and the other scattering events constitute unmodeled noise.

Since the image recovery algorithm tries to explain the detector measurements based on only coherent scattering, artifacts are introduced in the reconstruction. To reduce these artifacts in the momentum transfer region of interest (ROI), from 0.01 to 0.4 \si{\angstrom}$^{-1}$, we performed the reconstruction using a larger region extending to 0.6 \si{\angstrom}$^{-1}$ and later cropped to the ROI. Fig.~\ref{fig:reconstruction_results_MC_graphite} shows the results of estimating the mean detector photon counts and the object scatter density, using 20 iterations of the OSEM-type algorithm and the non-optimized (NO) and fully-optimized (AO) forward and backward models. Again, we see that the spatial distribution and the momentum transfer profile are recovered quickly and fairly accurately. Noticeable artifacts are observed in the spatial distribution at the corners of the reconstruction region, due to the mismatch between the model used for the Monte Carlo simulation and the analytical model. For the same reason, the reconstruction using the fully-optimized (AO) models, with 250 SABs, incidentally appears to fit the reference object better than the fully-optimized, with 2000 SABs, and the non-optimized models, as shown in Fig.~\ref{fig:mtp_profile_MC_graphite}.

\begin{figure*}[!htb]%
	\centering
	\subfigure[Measurements at detectors]{%
		\includegraphics[width=0.4\textwidth]{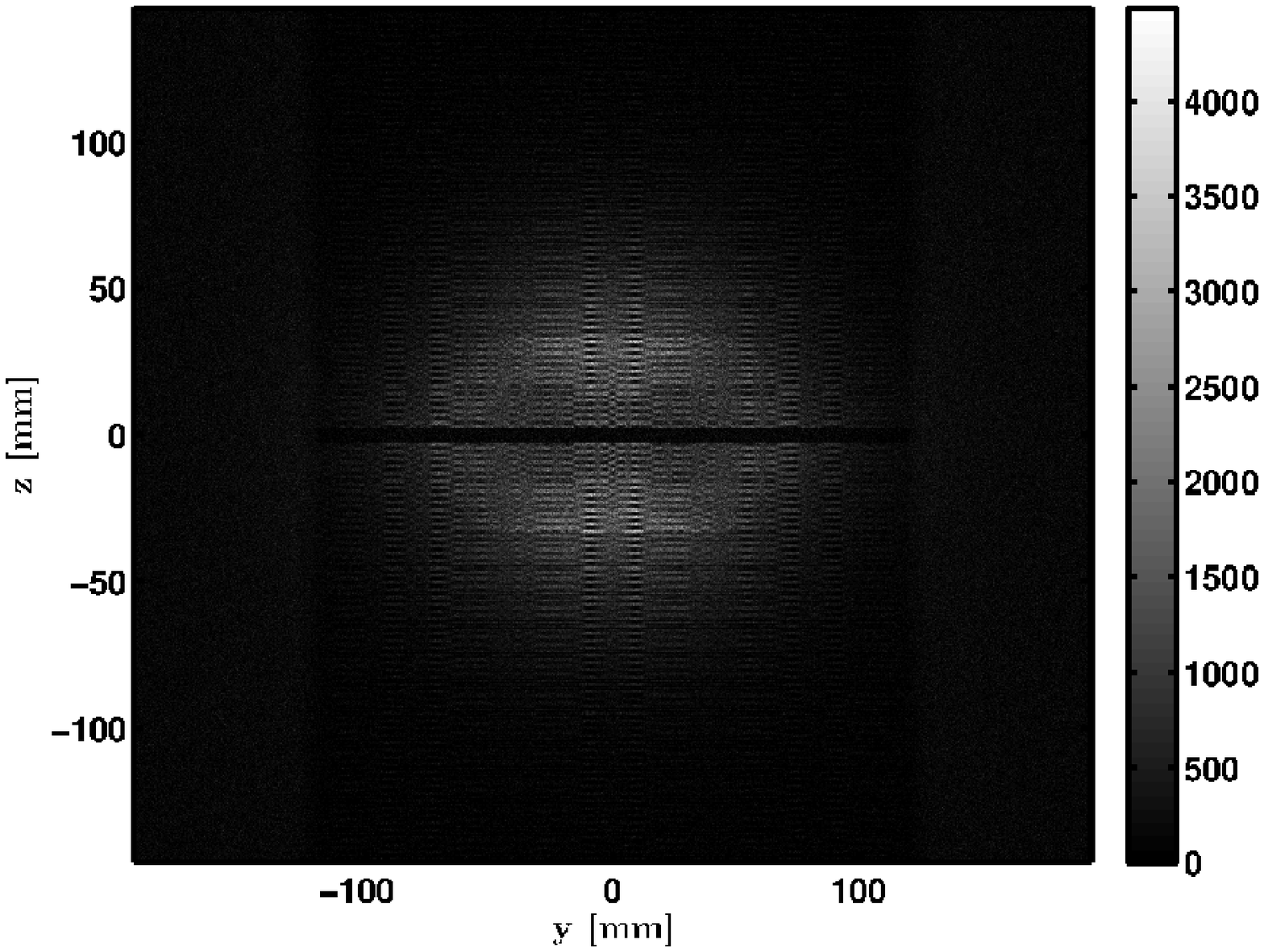}%
		\label{fig:simulated_data_MC_graphite}}%
	~
	\subfigure[Estimated mean measurements]{%
		\includegraphics[width=0.4\textwidth]{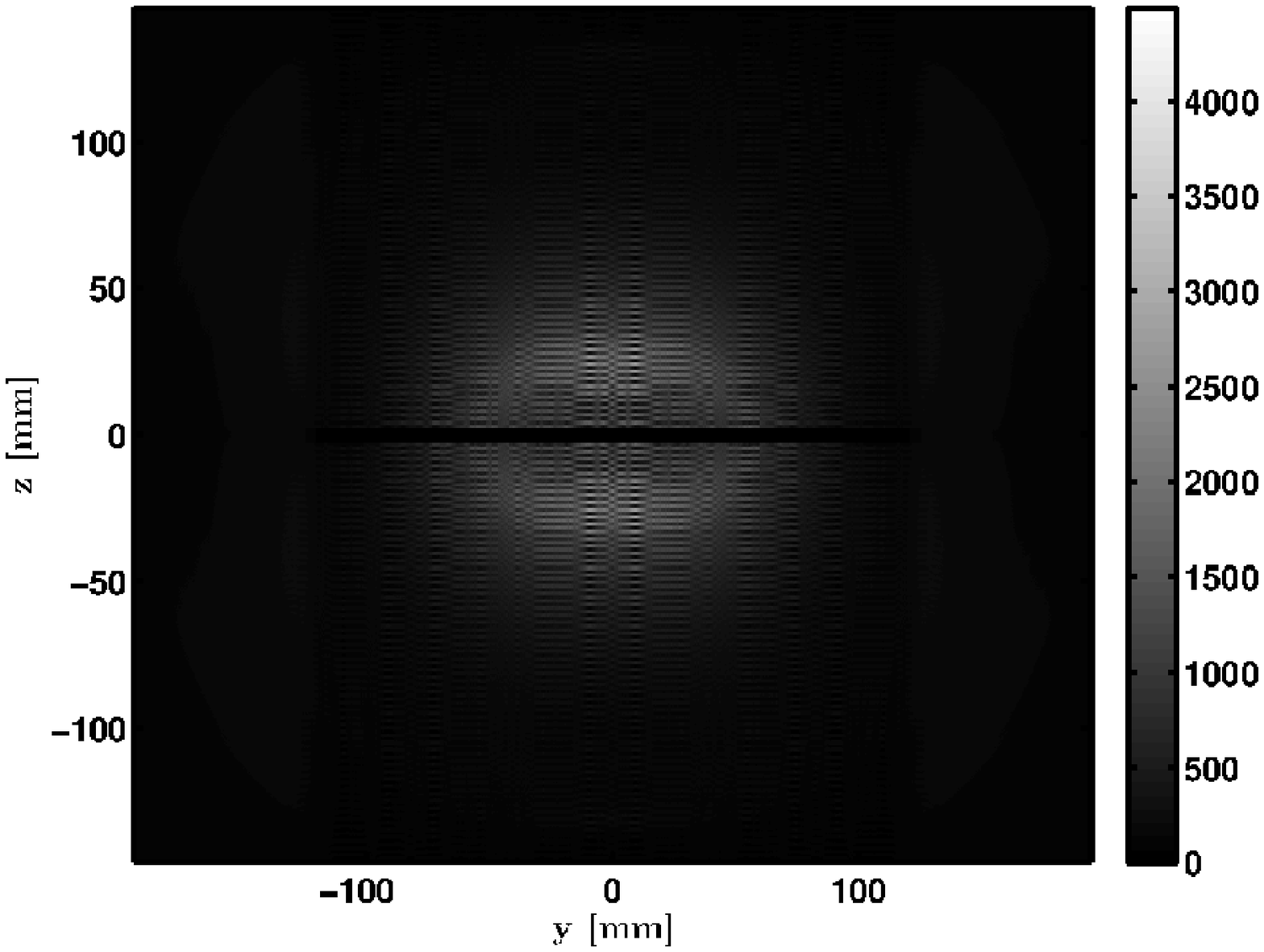}%
		\label{fig:estimated_data_MC_graphite}}%
	\\
	\subfigure[Estimated spatial distribution]{%
		\includegraphics[width=0.4\textwidth]{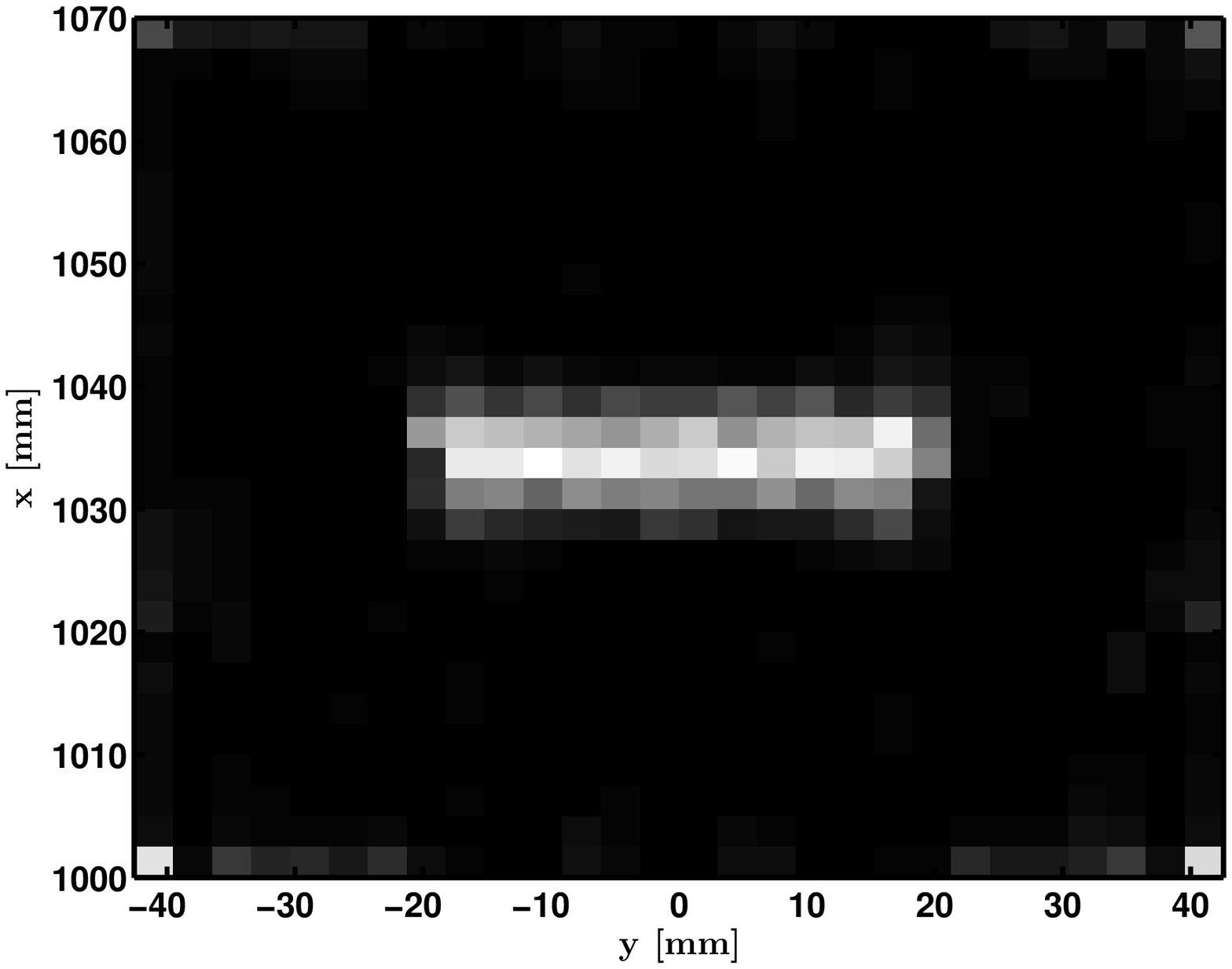}%
		\label{fig:estimated_spatial_dist_MC_graphite}}%
	~
	\subfigure[Normalized MTP for graphite]{%
		\includegraphics[width=0.4\textwidth]{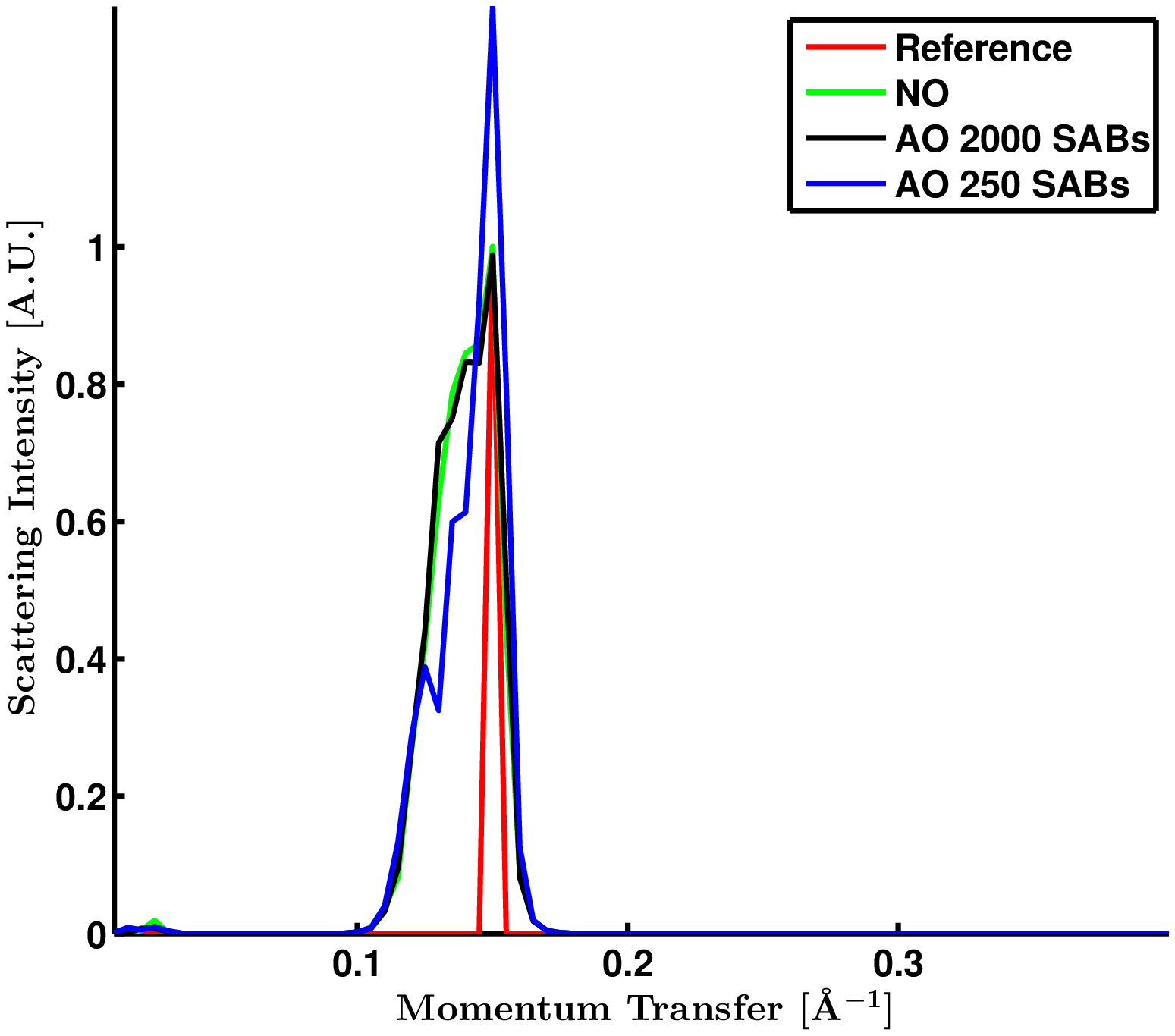}%
		\label{fig:mtp_profile_MC_graphite}}%
	\caption{Estimation of object scatter density from Monte Carlo measurements of a slice of graphite rectangular prism using ordered-subsets implementations in Algorithm~\ref{alg:image_update_OSEM_type}. The momentum transfer profile (MTP) was obtained by averaging the MTPs across the known location occupied by graphite. The estimated mean data was obtained by applying the fully-optimized forward operator, with 250 SABs, on the estimated object scatter density. See Fig.~\ref{fig:reconstruction_results_vial_configuration} for notation.}%
	\label{fig:reconstruction_results_MC_graphite}%
\end{figure*}

These results demonstrate that without prior knowledge of the location of scatterers in an object, the fan beam system, in conjunction with the OSEM-type algorithm and efficient implementations of the forward and backward models, can be used to efficiently estimate their scattering densities $f(x, y, q)$.

\section{Conclusions}\label{sec:conclusions}
The image recovery process relies heavily on a computational representation of the forward model for the data. We have identified and described three ways of significantly reducing the computational burden of the forward and backward models and the overall reconstruction algorithm, namely: scatter angle interpolation, symmetries in the system geometry, and balancing online and offline computations.

The forward and backward models and the corresponding algorithms were validated using analytic and Monte Carlo simulations. Speedups of about 146 and 32 were obtained in computing the forward and backward models using the ordered-subsets implementations, respectively. The spatial distribution and the momentum transfer profiles of the simulated objects were recovered fairly accurately using a few iterations of the ordered subset EM-type reconstruction algorithm.

The design of the coherent scatter imaging system, including the choice of the detector array, its pixel pitch, and placement relative to the source, and the secondary mask and its placement, influences the choices made in implementing the forward and backward models utilized in the reconstruction algorithm. For example, the use of different symmetry classes is dictated by their presence in the designed physical system.  Moreover, if the designed secondary mask and its placement fail to satisfy certain desirable properties that make the online computation of the mask modulation factor easy, it may be better to compute it offline, since its online computation may become the bottleneck in the computation of the forward model.

On the other hand, the need for efficient algorithms also affects several elements of the system design. For example, placing the center of the detector array so that the central ray from the source strikes it and adequate efforts ensuring proper alignment of the imaging components, permit efficient online implementations of the forward model and reconstruction algorithms. Also, having a detector with a smaller pixel pitch permits the potential use of larger translation symmetry step size, without significant deterioration of the recoverable object spatial resolution in the $y$ direction.

This paper describes the use of tomographic x-ray coherent scatter measurements in estimating volumetric scatter density, using a fan beam source distribution and a secondary mask. However, a subset of the methods we described for accelerating the computation of the forward and backward models and recovering the object scatter density can be easily carried over to other source distributions and more accurate forward models incorporating Compton scatter. For example, the scatter angle interpolation technique will apply to any x-ray coherent scatter imaging system with a polychromatic source.

\appendices
\section{Forward Model Derivation} \label{app:forward_model}
In this section of the Appendix, we describe the multiplicative factors that comprise the forward operator $H\left(\mathbf{r}^{\prime}, \mathbf{r}, q\right)$ given in Eq.~\ref{eq:forward_model_energy_integrating}.

The x-ray source is assumed to be polychromatic. Whatever source filtering and energy-dependent factors such as detector efficiency, present in the physical system are included in the effective source model whose energy-dependent flux is $\Phi(E)$, where the x-ray energy $E$ is given in keV. A photon, with energy $E$, incident on the object at $\mathbf{r}$ is scattered through an angle $\theta$, passes through or gets blocked by the coded aperture, and is measured by a detector pixel located at $\mathbf{r}^{\prime}$. For an energy-sensitive detector array, the number of photons detected is given by 

\begin{equation}
\label{eq:energy_sensitive}
\begin{aligned}
g(E, \mathbf{r}^{\prime}) &= \int \int \, \Phi(E) w_o \Delta\Omega_{s o} \frac{d\sigma_{\mbox{coh}}}{d\Omega} \Delta\Omega_{o d} A_{s o}(E, \mathbf{r}) \\
& A_{o d}(E, \mathbf{r}, \mathbf{r}^{\prime}) T(\mathbf{r}, \mathbf{r}^{\prime}) \delta\left(E - \frac{h c q}{\sin\left(\frac{\theta}{2}\right)}\right) d\mathbf{r} \, dq,
\end{aligned}
\end{equation}
\noindent{where}
\begin{eqnarray}
\frac{d\sigma_{\mbox{coh}}}{d\Omega} & = & \frac{d\sigma_{\mbox{Thompson}}}{d\Omega} f(\mathbf{r}, q), \nonumber \\
\frac{d\sigma_{\mbox{Thompson}}}{d\Omega} & = & \frac{r_e^2}{2} \left(1 + \cos^2\theta\right).\nonumber
\end{eqnarray}
\noindent{This is a modified form of the forward model in \cite{Greenberg2013snapshot} which separates the attenuation factor and solid angle $\Delta\Omega$ into two parts; from source to object and from object to detector.} 

$\Delta\Omega_{s o}$ is the solid angle subtended by the face of the object voxel illuminated by the source photon while $\Delta\Omega_{o d}$ is the solid angle subtended by the face of the detector pixel illuminated by the scattered photon. Although a fan beam is effectively planar, we assume that each pixel in the fan beam within the object space is effectively a voxel, with out-of-plane height given by the thickness of the fan beam. We assume that the thickness of the fan beam is fixed across the entire object. $w_o$ is the width of each object voxel in the direction parallel to the normal vector to the face of the object voxel illuminated by photons. $A_{s o}(E, \mathbf{r})$ is the attenuation of the photon from the source to the object, $A_{o d}(E, \mathbf{r}, \mathbf{r}^{\prime})$ is the attenuation of the scattered photon from the object to the detector, and $T(\mathbf{r}, \mathbf{r}^{\prime})$ is the transmission of the scattered photon through the coded aperture from the object to the detector. $d\sigma_{\mbox{Thompson}}/d\Omega$ is the Thompson form of the differential cross section of a free electron \cite{Harding1987}, $d\sigma_{\mbox{coh}}/d\Omega$ is the differential cross section describing the scatter of an x-ray with energy $E$ into a given solid angle, $r_e$ is the classical electron radius, and $f(\mathbf{r}, q)$ is the square of the coherent scatter form factor \cite{Harding1987, Greenberg2013snapshot}. We call $f(\mathbf{r}, q)$ the object scatter density.

The momentum transfer parameter $q$ is related to the x-ray energy $E$ through the equation \cite{Harding1987}
\begin{equation}
\label{eq:bragg}
q = \frac{E \sin\left(\frac{\theta}{2}\right)}{h c},
\end{equation}
\noindent{where $h$ is Planck's constant, $c$ is the speed of light, and $h c = 12.3984193$ keV \si{\angstrom}. Equation~\ref{eq:bragg} is often referred to as Bragg's law \cite{Greenberg2013snapshot}}.

We assume that the detector array is energy-integrating, so that photons of all energies are accumulated  at detector location $\mathbf{r}^{\prime}$. The energy-integrated flux is then given as 
\begin{align}
g(\mathbf{r}^{\prime}) &= \int \int \int \, \Phi(E) w_o \Delta\Omega_{s o} \frac{d\sigma_{\mbox{coh}}}{d\Omega} \Delta\Omega_{o d} \nonumber \\
&\quad\quad T(\mathbf{r}, \mathbf{r}^{\prime}) \delta\left(E - \frac{h c q}{\sin\left(\frac{\theta}{2}\right)}\right) \, d\mathbf{r} \, dq, \, dE, \label{eq:energy_integrating} \\
&= \int \int \, \left(\int \Phi(E) \delta\left(E - \frac{h c q}{\sin\left(\frac{\theta}{2}\right)}\right) \, dE \right) w_o \Delta\Omega_{s o} \nonumber \\
&\quad\quad \frac{r_e^2}{2} \left(1 + \cos^2\theta\right) f(\mathbf{r}, q) \Delta\Omega_{o d} T(\mathbf{r}, \mathbf{r}^{\prime}) \, d\mathbf{r} \, dq,\nonumber
\end{align}
\noindent{where we assume that the object is weakly attenuating so that $A_{s o}$ and $A_{o d}$ can be ignored.}

Bragg's law acts to pick out $(\theta, q)$ pairs over a range of energies corresponding to the intersection of the incident source energies and the energies the detector is sensitive to. These energies are represented by the domain of $\Phi(E)$. Using a Dirac delta formulation, for a point detector, the integral over $E$ may be simplified as  
\begin{equation}
\label{eq:dirac_delta_source}
\int \Phi(E)\delta\left(E - \frac{h c q}{\sin\left(\frac{\theta}{2}\right)}\right) \, dE = \Phi\left(\frac{h c q}{\sin\left(\frac{\theta}{2}\right)}\right).
\end{equation}

Due to a finite detector pixel size, there is a spread of scatter angles $\Delta\theta$ corresponding to each pair of object and detector points. This corresponds to a spread in Bragg's energy $\Delta E$, for each momentum transfer value. To account for the spread in Bragg's energy, we modify the integral in Eq.~\ref{eq:dirac_delta_source} by scaling the Dirac delta as in
\begin{equation}
\label{eq:scaled_dirac_delta_source}
\int \Phi(E)\delta\left(\frac{E - \frac{h c q}{\sin\left(\frac{\theta}{2}\right)}}{\Delta E}\right) \, dE = \Phi\left(\frac{h c q}{\sin\left(\frac{\theta}{2}\right)}\right) \Delta E.
\end{equation}
Using Bragg's law, Eq.~\ref{eq:bragg}, the spread in Bragg's energy is
\begin{equation}
\nonumber
\Delta E = \frac{h c q \cos\left(\frac{\theta}{2}\right)}{2 \sin^2\left(\frac{\theta}{2}\right)} \Delta\theta.
\end{equation}
\noindent{An estimate of $\Delta\theta$ is the angle between the scatter vectors from the object point to the midpoint of the two edges of  the finite detector pixel along the $z$ direction. This choice is in tune with the symmetry cases identified in Section~\ref{sec:computation_forward_model}.} 

The energy-integrated flux at detector point $\mathbf{r}^{\prime}$ is then given as 
\begin{equation}
\label{eq:energy_integrating_finite_detector}
\begin{aligned}
g(\mathbf{r}^{\prime}) &= \int \int \, \Phi\left(\frac{h c q}{\sin\left(\frac{\theta}{2}\right)}\right)  \frac{h c q \cos\left(\frac{\theta}{2}\right)}{2 \sin^2\left(\frac{\theta}{2}\right)} \Delta\theta w_o \Delta\Omega_{s o} \\
 & \frac{r_e^2}{2} \left(1 + \cos^2\theta\right) f(\mathbf{r}, q) \Delta\Omega_{o d} T(\mathbf{r}, \mathbf{r}^{\prime}) \, d\mathbf{r} \, dq.
\end{aligned}
\end{equation}

The source-to-object differential solid angle is given as
\begin{equation}
\label{eq:so_solid_angle}
\Delta\Omega_{s o} = G_{s o}(\mathbf{r}) A_o,
\end{equation}
\noindent{where}
\begin{equation}
\label{eq:so_geometry}
G_{s o}(\mathbf{r}) = \frac{\left|\hat{\mathbf{n}}_o \cdot \hat{\mathbf{r}}\right|}{r^2}.
\end{equation}
\noindent{$G_{s o}(\mathbf{r})$ is called the source-to-object geometry factor. $\hat{\mathbf{n}}_o$ which is illustrated in Fig.~\ref{fig:fanbeam_schematic}, is a unit normal vector to the voxel containing the object point, $\hat{\mathbf{r}}$ is the unit vector in the direction from the source to the object point, $r = |\mathbf{r}|$ is the magnitude of the vector from the source to the object point, and $A_o$ is the area of the illuminated face of the object voxel. Note that $A_o$ and $w_o$ multiply to give the volume of the object voxel $V_o$. We assume that the object space is uniformly sampled, so that the voxels have the same volume.} The front face of the object voxel illuminated by a ray is oriented such that $\hat{\mathbf{n}}_o$ is parallel to the direction of the fan beam's central ray (the x-axis). As such, $G_{s o}$ can be simplified to
\begin{equation}
\label{eq:so_geometry_simplified}
G_{s o}(\mathbf{r}) = \frac{x}{(x^2 + y^2)^{1.5}}. 
\end{equation}  

The object-to-detector differential solid angle is given as
\begin{equation}
\label{eq:od_solid_angle}
\Delta\Omega_{o d} = G_{o d}(\mathbf{s}) A_d,
\end{equation}
\noindent{where}
\begin{equation}
\label{eq:od_geometry}
G_{o d}(\mathbf{s}) = \frac{\left|\hat{\mathbf{n}}_d \cdot \hat{\mathbf{s}}\right|}{s^2}.
\end{equation}
\noindent{We call $G_{o d}(\mathbf{s})$ the object-to-detector geometry factor. $\hat{\mathbf{n}}_d$ is a unit normal vector to the detector array, $\hat{\mathbf{s}}$ is the unit vector in the direction of the scatter vector $\mathbf{s}$, from the object point to the detector point. $s = |\mathbf{s}|$ is the magnitude of the scatter vector and $A_d$ is the area of the detector pixel. We also assume that the detector pixels have the same area.} 

The secondary mask is situated between the object and the detector array as illustrated in Fig.~\ref{fig:fanbeam_schematic}. It is assumed to lie on a plane parallel to the detector array, centered at the fan beam's central ray. There is a beam-stop on the secondary mask and in front of the fan beam \cite{Maccabe2013snapshot} to attenuate the transmitted rays. This secondary mask spatially modulates the scattered x-ray flux. Given the object point and detector point, the intersection of the scatter vector and the aperture is defined. The geometric factor $T(\mathbf{r}, \mathbf{r}^{\prime})$ represents the resulting spatial modulation of the scattered x-ray flux. For planar masks, in practice, we estimate the values of the binary secondary mask image using a calibration scan with a flood illumination. The resulting measured image is rescaled from the detector plane back to the secondary mask plane and then binarized. The intersection of each scattered ray with the secondary mask plane provides a point of index into the binary in-plane mask image. The resulting binary values give $T(\mathbf{r}, \mathbf{r}^{\prime})$.

Using the geometric factors defined above and the fact that $V_o = A_o w_o$ and $A_d$ are constants, the energy-integrated flux at detector point $\mathbf{r}^{\prime}$ given in Eq.~\ref{eq:energy_integrating_finite_detector} can be recast as  
\begin{align}
g(\mathbf{r}^{\prime}) &= \int \int \, V_o A_d G_{s o}(\mathbf{r}) G_{o d}(\mathbf{s}) T(\mathbf{r}, \mathbf{r}^{\prime}) \Delta\theta \nonumber\\
&\quad\quad \Phi\left(\frac{h c q}{\sin\left(\frac{\theta}{2}\right)}\right)  \frac{h c q \cos\left(\frac{\theta}{2}\right)}{2 \sin^2\left(\frac{\theta}{2}\right)} \frac{r_e^2}{2} \left(1 + \cos^2\theta\right) \nonumber \\
&\quad\quad f(\mathbf{r}, q) \, d\mathbf{r} \, dq, \nonumber \\
&= \int \int \, C G_{s o}(\mathbf{r}) G_{o d}(\mathbf{s}) T(\mathbf{r}, \mathbf{r}^{\prime}) \Delta\theta \nonumber \\
&\quad\quad S(\theta, q) f(\mathbf{r}, q) \, d\mathbf{r} \, dq, \label{eq:energy_integrating_finite_detector_recast}
\end{align}
\noindent{where}
\begin{equation}
\label{eq:spectral_factor}
S(\theta, q) = \frac{q \left(1 + \cos^2\theta\right) \cos\left(\frac{\theta}{2}\right)}{\sin^2\left(\frac{\theta}{2}\right)} \Phi\left(\frac{h c q}{\sin\left(\frac{\theta}{2}\right)}\right),
\end{equation}
\noindent{is a spectral factor and $C = 0.25 h c V_o A_d r_e^2$ is a normalization constant that can be obtained by a calibration scan and includes scalar factors such as the area of a detector pixel and the volume of an object voxel.}

\section{Reconstruction Algorithm Details} \label{app:reconstruction_algorithm_details}
To solve optimization problem \ref{op:orig_problem}, we consider a sequence of simpler optimization problems obtained by lifting the objective function around an expansion point (e.g. the previous image estimate). In particular, we lift the objective function to obtain a surrogate objective function which is fully separable with respect to the image parameters. This choice allows us to utilize the optimized code for on-the-fly forward and backward projection. 

We adopt the EM surrogate function \cite{DePierro1995} for the data-fit term $L(\mathbf{f})$ and utilize De Pierro's convexity trick \cite{Erdogan1999B} for the regularization term. The EM surrogate function for the data-fit term is given as
\begin{equation}
\label{eq:surrogate_data_fit}
\hat{L}(\mathbf{f}) = \sum_{i = 1}^{I}\sum_{j = 1}^{J}\left[-\frac{y_i H_{i, j} \hat{f}_j}{\hat{l}_i + r_i} \ln\left(\frac{f_j}{\hat{f}_j} \left(\hat{l}_i + r_i\right)\right) + H_{i, j} f_j \right],
\end{equation}
\noindent{where $\hat{l}_i = \sum_{j\prime = 1}^{J} H_{i, j\prime} \hat{f}_{j\prime}$ and $\hat{\mathbf{f}}$ is the previous image estimate.} 

A surrogate function for the regularization term, using De Pierro's convexity trick \cite{Erdogan1999B} is given as 
\begin{equation}
\label{eq:surrogate_regularizer}
\begin{aligned}
\tilde{R}(\mathbf{f}) &= \sum_{j = 1}^{J} \sum_{k \in \mathcal{N}_{j}} 0.5 w_{j, k} \left[\psi_{\delta}\left(2 f_j - \hat{f}_j - \hat{f}_k\right) \right. \\
& \left. \quad + \psi_{\delta}\left(2 f_k - \hat{f}_j - \hat{f}_k\right) \right].
\end{aligned}
\end{equation}

Using a quadratic surrogate for the edge-preserving potential function of the form \cite{Huber1981, Erdogan1999}
\begin{equation}
\hat{\psi}_{\delta}(x) = \psi_{\delta}(\hat{x}) + \dot{\psi}_{\delta}(\hat{x})(x - \hat{x}) + \frac{1}{2} \omega_{\psi_{\delta}}(\hat{x})(x - \hat{x})^2, \nonumber
\end{equation}
\noindent{in Eq.~\ref{eq:surrogate_regularizer}, we obtain the following separable quadratic surrogate function to the regularization term}
\begin{equation}
\label{eq:quadratic_surrogate_regularizer}
\begin{aligned}
\hat{R}(\mathbf{f}) &= \sum_{j = 1}^{J} \sum_{k \in \mathcal{N}_{j}} 0.5 w_{j, k} \left[\hat{\psi}_{\delta}\left(2 f_j - \hat{f}_j - \hat{f}_k\right) \right. \\
& \left. \quad + \hat{\psi}_{\delta}\left(2 f_k - \hat{f}_j - \hat{f}_k\right) \right].
\end{aligned}
\end{equation}
\noindent{where $\hat{x}$ is the expansion point, $\dot{\psi}_{\delta}(\cdot)$ is the first derivative of $\psi_{\delta}(\cdot)$ and $\omega_{\psi_{\delta}}(x) \triangleq \dot{\psi}_{\delta}(x)/x$ is the curvature of the quadratic.}

The modified sequence of constrained convex optimization problems of interest is then
\begin{equation}
\label{op:modified_problems}
\mathbf{f}^{t + 1} = \argmin_{\mathbf{f} \geq \mathbf{0}} \hat{J}(\mathbf{f}),
\end{equation}
\noindent{with $\hat{J}(\mathbf{f}) = \hat{L}(\mathbf{f}) + \beta \hat{R}(\mathbf{f})$ and $\hat{\mathbf{f}} \triangleq \mathbf{f}^{t}$.}

Taking the derivative of $\hat{J}(\mathbf{f})$ with respect to $f_j$, we obtain the gradient ($\forall j$)
\begin{eqnarray}
\label{eq:grad_sep_obj}
\frac{\partial \hat{J}(\mathbf{f})}{\partial f_j} & = & -\frac{\hat{f}_j}{f_j} b_j^{(2)} + b_j^{(1)} \nonumber \\
& & + \beta\left[\sum_{k \in \mathcal{N}_j} w_{j, k} \dot{\hat{\psi}}_{\delta}\left(2 f_j - \hat{f}_j - \hat{f}_k\right)\right. \nonumber \\
& & \left. + \sum_{m \in \mathcal{N}_j^{\mathrm{b}}} w_{m, j} \dot{\hat{\psi}}_{\delta}\left(2 f_j - \hat{f}_j - \hat{f}_m\right)\right] \nonumber \\
& = & -\frac{\hat{f}_j}{f_j} b_j^{(2)} + b_j^{(1)} + \beta\left[f_j\left(b_j^{(3)} + b_j^{(4)}\right) \right. \\
& & \left. - \hat{f}_j\left(b_j^{(3)} + b_j^{(4)}\right) + b_j^{(5)} + b_j^{(6)}\right] \nonumber
\end{eqnarray}
\noindent{where}
\begin{eqnarray}
b_j^{(1)} & = & \sum_{i = 1}^{I} H_{i, j}, \nonumber \\
b_j^{(2)} & = & \sum_{i = 1}^{I} H_{i, j} \left(\frac{y_i}{\hat{l}_i + r_i}\right), \nonumber \\
b_j^{(3)} & = & 2\sum_{k \in \mathcal{N}_j} w_{j, k} \omega_{\psi_{\delta}}\left(\hat{f}_j - \hat{f}_k\right), \nonumber \\
b_j^{(4)} & = & 2\sum_{m \in \mathcal{N}_j^{\mathrm{b}}} w_{m, j} \omega_{\psi_{\delta}}\left(\hat{f}_j - \hat{f}_m\right), \nonumber \\
b_j^{(5)} & = & \sum_{k \in \mathcal{N}_j} w_{j, k} \dot{\psi}_{\delta}\left(\hat{f}_j - \hat{f}_k\right), \nonumber \\
b_j^{(6)} & = & \sum_{m \in \mathcal{N}_j^{\mathrm{b}}} w_{m, j} \dot{\psi}_{\delta}\left(\hat{f}_j - \hat{f}_m\right), \nonumber
\end{eqnarray}
\noindent{$\dot{\hat{\psi}}_{\delta}(x) \triangleq d\hat{\psi}_{\delta}(x)/dx$ and $\mathcal{N}_j^{\mathrm{b}}$ is the set of image voxels that have voxel $j$ as a neighbor}. For a symmetric neighborhood structure, $\mathcal{N}_j^{\mathrm{b}} = \mathcal{N}_j$ and $w_{j, k} = w_{k, j}$. Note that $\hat{l}_i$ is obtained by applying the forward operator on the previous image estimate. Moreover, $b_j^{(1)}$ and $b_j^{(2)}$ are obtained by applying the backward operator on a detector image of all ones and the ratio of the actual measurement to the predicted measurement, respectively.

Using Lemma~\ref{lem:lemma1} (see Appendix~\ref{app:lemma}), the gradient equation $\frac{\partial \hat{J}(\mathbf{f})}{\partial f_j} = 0$ admits the closed-form solution 
\begin{equation}
\label{eq:f_update}
f_j = \left\{
\begin{array}{ll}
\frac{-\chi_j^{(2)} + \sqrt{\left(\chi_j^{(2)}\right)^2 + 4 \chi_j^{(1)} \chi_j^{(3)}}}{2 \chi_j^{(1)}} & \text{if $\chi_j^{(1)} > 0$} \\
\frac{\chi_j^{(3)}}{\chi_j^{(2)}} & \text{if $\chi_j^{(1)} = 0$}
\end{array} \right. ,
\end{equation}
\noindent{where}
\begin{eqnarray}
\chi_j^{(1)} & = & \beta\left(b_j^{(3)} + b_j^{(4)}\right), \nonumber \\
\chi_j^{(2)} & = & b_j^{(1)} + \beta\left[-\hat{f}_j\left(b_j^{(3)} + b_j^{(4)}\right) + b_j^{(5)} + b_j^{(6)}\right], \nonumber \\
\chi_j^{(3)} & = & \hat{f}_j b_j^{(2)}. \nonumber
\end{eqnarray}

\section{Lemma Statement and Proof} \label{app:lemma}

\begin{lemma}
\label{lem:lemma1}
The extended function
\begin{equation}
\label{eq:1d_fun}
f(x) = \frac{1}{2} ax^2 + bx - c \ln(x) + d, \qquad x \in [0, +\infty],
\end{equation}

\noindent achieves a unique minimum at
\begin{equation}
\label{eq:overall_solution_set}
x^{*} = \left\{
\begin{array}{ll}
\frac{-b + \sqrt{b^2 + 4ac}}{2a} & \text{if $a > 0$} \\
\frac{c}{b} & \text{if $a = 0, b > 0$} \\
+\infty & \text{if $a = 0, b \leq 0$} 
\end{array} \right. ,
\end{equation}

\noindent for $a \geq 0, c \geq 0, \mbox{and }b, d \in\mathbb{R}$, where we assume $0\ln(0) = 0 \mbox{ and } ln(0) = -\infty$.
\end{lemma}

\begin{proof}
The first and second derivatives of $f(x)$ are (for $x > 0$)
\begin{equation}
\nonumber
\frac{df(x)}{dx} = ax + b - \frac{c}{x},
\end{equation}
\noindent{and }
\begin{equation}
\nonumber
\frac{d^2f(x)}{dx^2} = a + \frac{c}{x^2},
\end{equation}
\noindent{respectively.}

Since $\frac{d^2f(x)}{dx^2} \geq 0$, $f(x)$ is a convex function. Consider the cases of $a > 0$, $a = 0, b > 0$, and $a = 0, b \leq 0$

\begin{case}
$a > 0$

Since $\frac{d^2f(x)}{dx^2} > 0$ and $\lim_{x \rightarrow \infty} \frac{d^2f(x)}{dx^2} = a > 0$, $f(x)$ is strongly convex. Thus, $f(x)$ admits a unique minimizer.

The critical point of $f(x)$ occurs at the solution of the gradient equation $\frac{df(x)}{dx} = 0$ that lies within the domain of $f(x)$. The only endpoint of the domain of $f(x)$ is at $x = 0$. The unique minimizer occurs either at $x = 0$ or the critical point.

Consider the sub-cases of $c = 0$ and $c > 0$. When $c = 0$, the gradient equation becomes $ax + b = 0$, so that $x = -\frac{b}{a}$ is a critical point and the unique minimizer, if $b \leq 0$. If $b > 0$, the endpoint $x = 0$ is the unique minimizer. 

Thus, when $a > 0 \mbox{and }c = 0$,
\begin{equation}
\nonumber
x^{*} = \left\{
\begin{array}{ll}
-\frac{b}{a} & \text{if $b \leq 0$} \\
0 & \text{if $b > 0$}
\end{array} \right. .
\end{equation}

When $c > 0$, $f(0) = +\infty$, so that the endpoint $x = 0$ is not an absolute minimum.

Consider the gradient equation (with $x > 0$)
\begin{equation}
\label{eq:1d_grad_eqn}
ax + b -\frac{c}{x} = 0.
\end{equation}

Multiplying both sides of Eq.~\ref{eq:1d_grad_eqn} by $x$, we get the quadratic equation
\begin{equation}
\label{eq:1d_quad_eqn}
ax^2 + bx - c = 0.
\end{equation}

The discriminant of the quadratic expression is $D = b^2 + 4ac > 0$. In addition, since $4ac > 0$, $|b| < \sqrt{b^2 + 4ac}$, so that Eq.~\ref{eq:1d_quad_eqn} has exactly two solutions with opposite signs. Since $x$ must be non-negative, the only valid solution to Eq.~\ref{eq:1d_grad_eqn} and \ref{eq:1d_quad_eqn} and the unique minimizer is
\begin{equation}
\label{eq:1d_quad_soln}
x^{*} = \frac{-b + \sqrt{b^2 + 4ac}}{2a}.
\end{equation}

Equation~\ref{eq:1d_quad_soln} also handles the sub-case of $c = 0$. For that case, $x^{*} = \frac{-b + \sqrt{b^2}}{2a} = \frac{-b + |b|}{2a}$. When $b \leq 0$, $|b| = -b$, so that $x^{*} = -\frac{b}{a}$. When $b > 0$, $|b| = b$, and $x^{*} = 0$.

Thus, $x^{*}$, as given by Eq.~\ref{eq:1d_quad_soln}, is the unique minimizer of $f(x)$, when $a > 0$.
\end{case}

\begin{case}
$a = 0, b > 0$

In this case, the gradient equation becomes $b - c/x = 0$, with solution $x = \frac{c}{b} \geq 0$. Since at the endpoint of the domain of $f$, $f(0) = +\infty$, $x^{*} = \frac{c}{b}$ is the unique minimizer of $f(x)$.
\end{case}

\begin{case}
$a = 0, b \leq 0$

In this case, the gradient function is also $b - c/x$. However, the function is strictly decreasing and unbounded below, so that the minimum of $-\infty$ is achieved at $x^{*} = +\infty$.
\end{case}

Thus, $x^{*}$ in Eq.~\ref{eq:overall_solution_set} is the unique minimizer of the function $f(x)$ given in Eq.~\ref{eq:1d_fun}.
\end{proof}

\bibliography{./References}
\bibliographystyle{IEEEtran}
\end{document}